\documentclass{article}

\usepackage{microtype}
\usepackage{graphicx}
\usepackage{subfigure}
\usepackage{booktabs} 

\usepackage[utf8]{inputenc}
\usepackage{natbib}
\usepackage{xcolor}
\usepackage{enumitem}

\usepackage{amsmath}
\usepackage{amssymb}
\usepackage{mathtools}
\usepackage{amsthm}
\usepackage{hyperref}

\usepackage[accepted]{icml2024}

\usepackage{amsmath,amsfonts,bm}

\def\eqref#1{equation~\ref{#1}}

\def\1{\bm{1}}

\def\rmC{{\mathbf{C}}}

\def\rmE{{\mathbf{E}}}

\def\rmG{{\mathbf{G}}}
\def\rmH{{\mathbf{H}}}

\def\rmO{{\mathbf{O}}}
\def\rmP{{\mathbf{P}}}

\def\rmV{{\mathbf{V}}}

\def\va{{\bm{a}}}
\def\vb{{\bm{b}}}

\DeclareMathAlphabet{\mathsfit}{\encodingdefault}{\sfdefault}{m}{sl}
\SetMathAlphabet{\mathsfit}{bold}{\encodingdefault}{\sfdefault}{bx}{n}

\def\gC{{\mathcal{C}}}
\def\gD{{\mathcal{D}}}

\def\gN{{\mathcal{N}}}

\def\gP{{\mathcal{P}}}

\def\gR{{\mathcal{R}}}

\def\gU{{\mathcal{U}}}

\def\gX{{\mathcal{X}}}
\def\gY{{\mathcal{Y}}}

\def\sU{{\mathbb{U}}}

\theoremstyle{plain}
\newtheorem{theorem}{Theorem}[section]

\theoremstyle{definition}

\theoremstyle{remark}

\icmltitlerunning{Parameter Estimation in DAGs from Incomplete Data via Optimal Transport}

\begin{document}
\twocolumn[
\icmltitle{Parameter Estimation in DAGs from Incomplete Data via Optimal Transport}

\icmlsetsymbol{equal}{*}

\begin{icmlauthorlist}
\icmlauthor{Vy Vo}{yyy,ccc}
\icmlauthor{Trung Le}{yyy}
\icmlauthor{Tung-Long Vuong}{yyy,vvv}
\icmlauthor{He Zhao}{ccc}
\icmlauthor{Edwin V. Bonilla}{ccc}
\icmlauthor{Dinh Phung}{yyy,vvv}
\end{icmlauthorlist}

\icmlaffiliation{yyy}{Monash University, Australia}
\icmlaffiliation{ccc}{CSIRO's Data61, Australia}
\icmlaffiliation{vvv}{VinAI Research, Vietnam}

\icmlcorrespondingauthor{Vy Vo}{v.vo@monash.edu}

\icmlkeywords{Machine Learning, ICML}
\vskip 0.3in
]

\printAffiliationsAndNotice{} 
\begin{abstract}
Estimating the parameters of a probabilistic directed graphical model from incomplete data is a long-standing challenge. This is because, in the presence of latent variables, both the likelihood function and posterior distribution are intractable without  assumptions about structural dependencies or model classes. While existing learning methods are fundamentally based on likelihood maximization, here we offer a new view of the parameter learning problem through the lens of optimal transport. This perspective licenses a general framework that operates on any directed graphs without making unrealistic assumptions on the posterior over the latent variables or resorting to variational approximations. We develop a theoretical framework and support it with extensive empirical evidence demonstrating the versatility and robustness of our approach. Across experiments, we show that not only can our method effectively recover the ground-truth parameters but it also performs comparably or better than competing baselines on downstream applications. 

\end{abstract}

\section{Introduction}\label{sect:intro}
Learning probabilistic directed graphical models (DGMs, also known as Bayesian networks) with latent variables is an ongoing challenge in machine learning and statistics. This paper focuses on parameter learning, i.e., estimating the parameters of a DGM with its structure known. Learning DGMs has a long history, dating back to classical indirect likelihood-maximization approaches such as expectation maximization \citep[EM,][]{Dempster1977}. Despite all its success stories, EM is known to suffer from local optima issues. More importantly, EM becomes inapplicable when the posterior distribution is intractable, which arises fairly often in practice.  Furthermore, EM is originally a batch algorithm, thereby converging slowly on large datasets \citep{liang2009online}. Subsequently, researchers have explored combining EM with approximate inference along with other strategies to improve efficiency \citep{wei1990monte,neal1998view,delyon1999convergence,beal2006variational,cappe2009line,liang2009online,neath2013convergence}. A large family of approximation algorithms based on variational inference \citep[VI,][]{jordan1999introduction,hoffman2013stochastic} have demonstrated tremendous potential, where the evidence lower bound (ELBO) is not only used for posterior approximation but also for point estimation of the model parameters. Such an approach has proved effective and robust to overfitting, especially when having a small number of parameters. VI has recently taken a leap forward by embracing amortized inference \citep{amos2022tutorial}, which performs black-box optimization in a considerably more efficient way.

Prior to parameter estimation, both EM and VI consist of an inference step which ultimately requires carrying out expectations of the commonly intractable posterior over the latent variables. In order to address this challenge, a large spectrum of methods have been proposed in the literature and we refer the reader to  \citet{pmlr-v130-ambrogioni21a} for an excellent discussion of these approaches. Here we characterize them between two extremes. At one extreme, restrictive assumptions about the structure (e.g., as in mean-field approximations) or the model class (e.g., using conjugate exponential families) must be made to simplify the task. At the other extreme, when no assumptions are made, most existing black-box methods exploit very little information about the structure of the known probabilistic model, e.g., in black-box and stochastic VI \citep{ranganath2014black,hoffman2013stochastic}, hierarchical approaches \citep{ranganath2016hierarchical} or normalizing flows \citep{Papamakarios2021}. Section \ref{sect:rwork} summarizes the progression of VI research towards this extreme. Since the ultimate goal of VI is posterior inference, parameter estimation has been treated as a by-product of the optimization process where the model parameters are cast as global latent variables. As the complexity of the graph increases, parameter estimation in VI becomes computationally challenging. 

A natural question arises as to whether one can learn the parameters of any DGMs with hidden nodes without explicitly solving inference nor assuming any structural independencies. In this work, we revisit the classic problem of learning graphical models from the viewpoint of optimal transport \citep[OT,][]{villani2009optimal}, which permits a scalable and general framework that addresses the above criterion.   

\paragraph{OT as an alternative to MLE.} Estimating the model parameters is essentially about learning a probability density from empirical data. EM and VI are fundamentally based on maximum likelihood estimation (MLE), which amounts to, asymptotically, minimizing the KL (Kullback-Leibler) divergence between the true data and model distribution. We here propose to find a point estimate that minimizes the Wasserstein (WS) distance \citep{kantorovich1960mathematical} between these two distributions. The motivations of using WS distance to this problem are three-fold.

\textbf{First,} the measurability and consistency of the minimum Wasserstein estimators have been rigorously studied in prior research, notably in \citet{bassetti2006minimum,bernton2019parameter}. \textbf{Second,} WS distance is a metric, thus serving as a more sensible measure of distance between two distributions, especially those that are supported on low dimensional manifolds with negligible intersection of support, where standard metrics such as the KL or JS (Jensen-Shannon) divergences are either infinite or undefined \citep{peyre2017computational,Ambrogioni2018}. 

\textbf{Finally,} we substantiate the motivation of using OT for graphical learning with an additional experiment of learning GMM under mis-specifications. The task is to estimate the means of a mixture of two bi-variate Gaussian distributions with unit variance i.e., $\sigma_1 = \sigma_2 = \mathbf{I}$. 
The means of one Gaussian are $\mu_{11}, \mu_{12} \sim \gU(0,2)$ and the means of the other are $\mu_{21}, \mu_{22} \sim \gU(0,2)$. The mixture weight is $\pi \sim \gU(0.50, 0.70)$. Figure \ref{fig:misgmm} illustrates the mean absolute errors when (1) the variances are mis-specified at $\varepsilon_{\sigma} \mathbf{I}$ where $\varepsilon_{\sigma} \sim \gU(1,2)$; (2) the weights are mis-specified at $\varepsilon_{\pi} \sim \gU(0,1)$; (3) both are mis-specified. We compare EM with our proposed method that estimates the means by the minimum Wasserstein estimators. The figures show that while EM plateaus early on, our method continues to converge over training. This reaffirms that minimum Wasserstein estimators tend to be more reliable and robust under mis-specifications \citep{bernton2019parameter}. Despite the above desirable properties of the WS distance, the application of OT to estimating the parameters of a general DGM remains underexplored. Our work is proposed to fill in this gap. 

\paragraph{Contributions.} In this work, we introduce \textbf{OTP-DAG}, an \textbf{O}ptimal \textbf{T}ransport framework for \textbf{P}arameter Learning in \textbf{D}irected \textbf{A}cyclic \textbf{G}raphical models\footnote{Our code is published at \url{https://github.com/isVy08/OTP}.}. OTP-DAG is a flexible framework applicable to any type of variables and graphical structures. Our theoretical development renders a tractable formulation of the Wasserstein objective for models with latent variables, which is established as a generalization for the WAE model. We further provide empirical evidence demonstrating the versatility of our method on various graphical structures, where OTP-DAG is shown to successfully recover the ground-truth parameters and achieve comparable or better performance than competing methods across a range of downstream applications. 

\begin{figure}
    \centering\includegraphics[width=\linewidth]{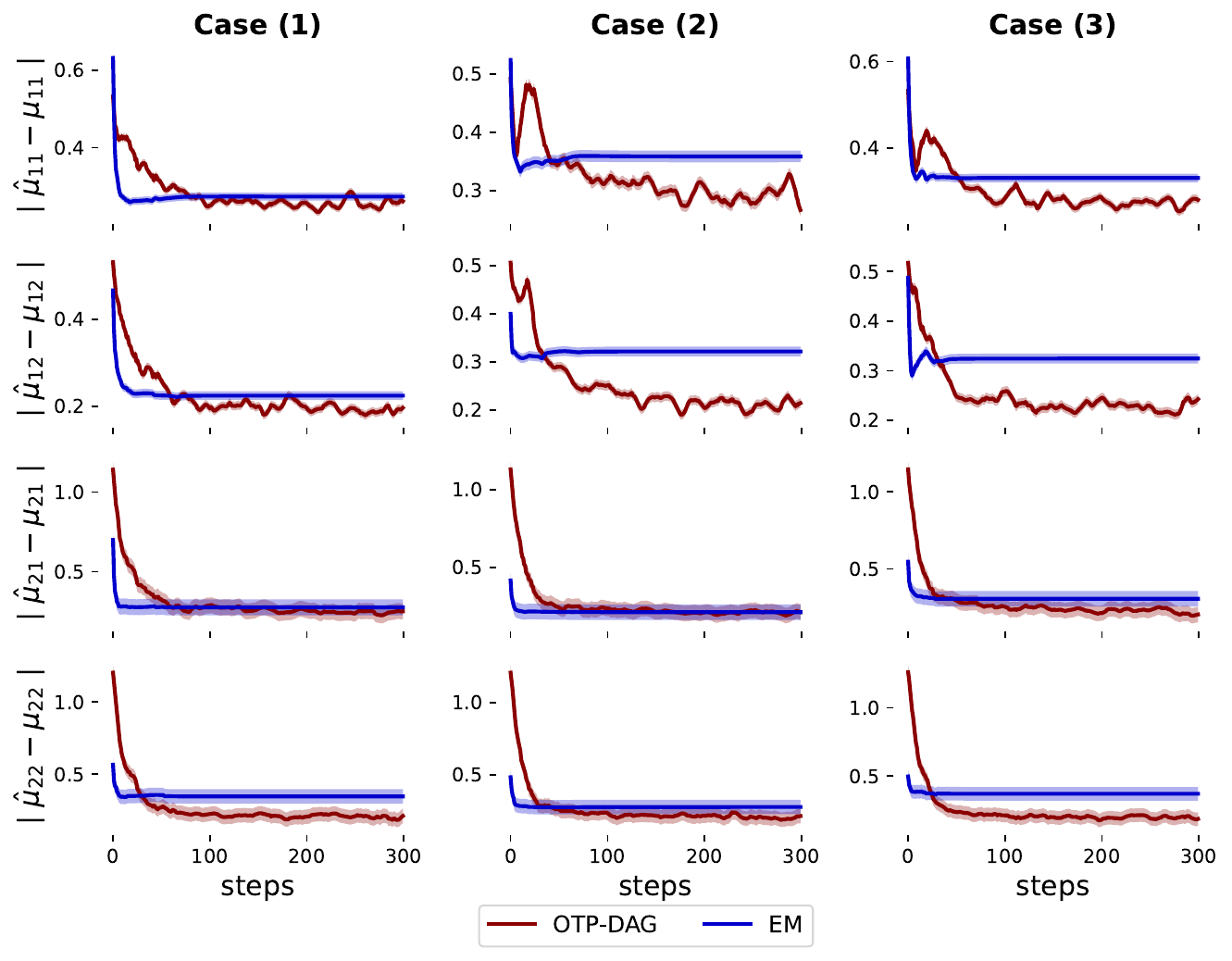}
    \caption{Visualization of mean absolute errors of the inferred means $\hat{\mu}$ and the true values $\mu$ for $300$ steps, averaged over $100$ simulations. $\mu_{ki}$ indicates the mean of the component $k$ at dimension $i$. The red line represents \textbf{our method \textcolor{red}{OTP-DAG}}. The blue line represents \textbf{\textcolor{blue}{EM}}. Three mis-specified cases are studied: \textbf{Case (1)} mis-specified variances, \textbf{Case (2)} mis-specified weights and \textbf{Case (3)} mis-specified both variances and weights.}
    \label{fig:misgmm}
\end{figure}

\section{Related work}\label{sect:rwork}
\paragraph{Variational Inference.} As part of parameter learning, both EM and VI entail an inference sub-process for posterior estimation. If the posteriors cannot be computed exactly, approximate inference is the go-to solution. In this section, we focus on variational algorithms and their computational challenges. Along this line, research efforts have concentrated on ensuring tractability of the ELBO via the mean-field assumption \citep{bishop2006pattern} and its relaxation known as structured mean field \citep{saul1995exploiting}. Scalability has been one of the main challenges facing the early VI formulations since the optimization is done on a per-sample basis. This has triggered the development of stochastic variational inference \citep[SVI,][]{hoffman2013stochastic,hoffman2015structured,foti2014stochastic,johnson2014stochastic,anandkumar2012method,anandkumar2014tensor} which applies stochastic optimization to solve VI objectives. 

Another line of work is collapsed VI that explicitly integrates out certain model parameters or latent variables in an analytic manner \citep{hensman2012fast,king2006fast,teh2006collapsed,lazaro2012overlapping}. Without a closed form, one could resort to Markov chain Monte Carlo \citep{gelfand1990sampling,gilks1995markov,hammersley2013monte}, which however tends to be slow. More accurate variational posteriors also exist, namely, through hierarchical variational models \citep{ranganath2016hierarchical}, implicit posteriors \citep{titsias2019unbiased,yin2018semi,molchanov2019doubly,titsias2019unbiased}, normalizing flows \citep{kingma2016improved}, or copulas \citep{tran2015copula}. 

To avoid computing ELBO analytically, one can obtain an unbiased gradient estimator using re-parameterization tricks \citep{ranganath2014black,xu2019variance}.  Extensions of VI to other divergence measures such as $\alpha-$divergence or $f-$divergence, also exist in \citet{li2016renyi,hernandez2016black,wan2020f}. A thorough review the above approaches can be found in \citet[][\S6]{pmlr-v130-ambrogioni21a}. In the causal inference literature, a related direction is to learn both the graphical structure and parameters of the corresponding structural equation model \citep{yu2019dag,geffner2022deep}. These frameworks are often limited to additive noise models while assuming no latent confounders.

\paragraph{Optimal Transport.} Optimal transport (OT) studies the optimal transportation of mass from one distribution to another \cite{villani2009optimal}. Through the notion of Wasserstein distance, OT offers a geometrically meaningful distance between probability distributions, proving effectiveness in various machine learning domains ~\cite{huynh2020otlda,zhao2020neural,nguyen2021most,wanrepresenting2022,bui2021unified,nguyen2022cycle,vuong2023vector,ye2024ptarl,gao2024distribution,luong2024revisiting,vo2024optimal}. 

Particularly, there has been a surge in OT application to generative models, namely Wasserstein GANs \citep[WGAN,][]{adler2018banach,arjovsky2017wasserstein} and Wasserstein Auto-encoders \citep[WAE,][]{tolstikhin2017wasserstein}. Underlying WAE is basically a two-node graphical model with one observed node (i.e., the data) and one latent node (i.e., often a Gaussian prior). There has also been application of OT for learning Gaussian mixture models (GMM), which are also two-node graphical models where the latent variable (i.e., the mixture weight) is categorical. \citet{mena2020sinkhorn} proposes an algorithm named Sinkhorn EM using entropic OT loss which yields faster convergence rate than vanilla EM. \citet{kolouri2018sliced} studies the extension of the Wasserstein mean problem \cite{ho2017multilevel} to learn a GMM, showing that the Wasserstein energy landscape is smoother and less sensitive to the initial point than that of the negative log likelihood. 

\section{Preliminaries}\label{sect:background}
We first introduce the notations and basic concepts used throughout the paper. We reserve bold capital letters (e.g., $\rmG$) for notations related to graphs. We use calligraphic letters (e.g., $\gX$) for spaces, italic capital letters (e.g., $X$) for random variables, and lower case letters (e.g., $x$) for values.  

\paragraph{Directed Graphical Models.} A directed graph $\rmG = (\rmV, \rmE)$ consists of a set of nodes $\rmV$ and an edge set $\rmE \subseteq \rmV^2$  of ordered pairs of
nodes with $(v, v) \notin \rmE$ for any $v \in \rmV$ (one without self-loops). For a pair of nodes $i,j$ with $(i,j) \in \rmE$, there is an arrow pointing from $i$ to $j$ and we write $i \rightarrow j$. Two nodes $i$ and $j$ are adjacent if either $(i,j) \in \rmE$ or $(j,i) \in \rmE$. If there is an arrow from $i$ to $j$ then $i$ is a parent of $j$ and $j$ is a child of $i$. A Bayesian network structure $\rmG = (\rmV, \rmE)$ is a \textbf{directed acyclic graph} (DAG), in which the nodes represent random variables $X = [X_i]^{n}_{i=1}$ with index set $\rmV := \{1,...,n\}$. Let $\mathrm{PA}_{X_i}$ denote the set of variables associated with parents of node $i$ in $\rmG$.  
In this work, we tackle the classic problem of learning the parameters of a directed graph from \textit{partially observed data}. Let $\rmO \subseteq \rmV$ and $X_\rmO = [X_i]_{i \in \rmO}$ be the set of observed nodes and $\rmH := \rmV \backslash \rmO$ be the set of hidden nodes. 
Let $P_{\theta}$ and $P_{d}$ respectively denote the distribution induced by the graphical model and the empirical one induced by the \textit{complete} (yet unknown) data. Given a fixed graphical structure $\rmG$ and some set of i.i.d data points, we aim to find the point estimate $\theta^{*}$ that best fits the observed data $X_{\rmO}$. The conventional approach is to minimize the KL divergence between the model distribution and the \textit{empirical} data distribution over observed data i.e., $\textrm{KL}(P_d(X_{\rmO}), P_{\theta}(X_{\rmO}))$, which is equivalent to maximizing the likelihood  $P_{\theta}(X_{\rmO})$ w.r.t $\theta$. 
In the presence of latent variables, the marginal likelihood, given as $P_{\theta}(X_\rmO) = \int_{X_{\rmH}} P_{\theta}(X) dX_{\rmH}$, is generally intractable. 

\paragraph{Optimal transport.}  Let $\alpha = \sum^{n}_{j=1} a_j \delta_{x_j}$ be a discrete measure with weights $\va$ and particles $x_1, \cdots, x^n \in \gX$ where $\va \in \Delta^{n}$, a $(n-1)-$dimensional probability simplex. Let $\beta = \sum^{n}_{j=1} b_j \delta_{y_j}$ be another discrete measure defined similarly. The Kantorovich \citep{kantorovich2006problem} formulation of the OT distance between two discrete distributions $\alpha$ and $\beta$ is 
\begin{equation}\label{eq:kanto_disc}
    W_c \left(\alpha, \beta \right) := \inf_{\rmP \in \sU(\va,\vb)} \langle \rmC, \rmP \rangle,
\end{equation}
where $\langle \cdot, \cdot \rangle$ denotes the Frobenius dot-product; $\rmC \in \mathbb{R}^{n \times n}_{+}$ is the cost matrix of the transport; $\rmP \in \mathbb{R}^{n \times n}_{+}$ is the transport matrix/plan; $\sU(\va,\vb) := \left\{ \rmP \in \mathbb{R}^{n \times n}_{+} : \rmP \mathbf{1}_{n} = \va, \rmP \mathbf{1}_{n} = \vb \right\} $ is the transport polytope of $\va$ and $\vb$; $\mathbf{1}_{n}$ is the $n-$dimensional column of vector of ones. For arbitrary measures, Eq. (\ref{eq:kanto_disc}) can be generalized as 
\begin{equation}\label{eq:kanto_cont}
    W_c \left(\alpha; \beta\right) := \underset{\Gamma \sim \gP (X \sim \alpha, Y \sim \beta)}{\mathrm{inf}} \mathbb{E}_{(X, Y) \sim \Gamma} \bigl[ c(X,Y)\bigr],
\end{equation}
where the infimum is taken over the set of all joint distributions $\gP (X \sim \alpha, Y \sim \beta)$ with marginals $\alpha$ and $\beta$ respectively. $c: \gX \times \gY \mapsto \mathbb{R}$ is any measurable cost function. If $c(x,y) = D^p(x,y)$ is a distance for $p \le 1$, $W_p$, the $p$-root of $W_c$, is called the $p$-Wasserstein distance.  Finally, for a measurable map $T : \gX \mapsto \gY$, $T\#\alpha$ denotes the push-forward measure of $\alpha$ that, for any measurable set $B \subset \gY$, satisfies $T\#\alpha(B) = \alpha(T^{-1}(B))$. For discrete measures, the push-forward operation is $T\#\alpha = \sum^{n}_{j=1} a_j \delta_{T(x_j)}$.

\section{Optimal Transport for Learning Directed Graphical Models}\label{sect:method}
We consider a DAG $\rmG(\rmV, \rmE)$ over random variables $X = [X_i]^{n}_{i=1}$ that represents the data generative process of an underlying system. The system consists of $X$ as the set of endogenous variables and $U = \{U_i\}^{n}_{i=1}$ as the set of exogenous variables representing external factors affecting the system. Associated with every $X_i$ is an exogenous variable $U_i$ whose values are sampled from a prior distribution $P(U)$ independently from the other exogenous variables. For the purpose of theoretical development, our framework operates on an extended graph consisting of both endogenous and exogenous nodes (See Figure \ref{fig:main}). In the graph $\rmG$, $U_i$ is represented by a node with no ancestors that has an outgoing arrow towards  its associated endogenous variable $X_i$. Every distribution $P_{\theta_i} \big(X_i \vert \mathrm{PA}_{X_i} \big)$ can be reparameterized into a deterministic assignment
\begin{equation*}
  X_i := \psi_i \big(\mathrm{PA}_{X_i}, U_i\big), \ \text{for  } i = 1, ..., n.  
\end{equation*}

The ultimate goal is to estimate $\theta =\{\theta_i\}_{i=1}^{n}$ as the parameters of the set of deterministic functions $\psi = \{\psi_i\}_{i=1}^{n}$. We will use the notation $\psi_{\theta}$ to emphasize this connection from now on. Given the empirical data distribution $P_d(X_{\rmO})$ and the model distribution $P_{\theta}(X_{\rmO})$ over the observed set $\rmO$, the optimal transport goal is to find the parameter set $\theta$ that minimizes the cost of transport between these two distributions. The Kantorovich’s formulation of the problem is given by
\begin{equation}\label{eq:kanto}
    W_c \big(P_d; P_{\theta}\big) := \inf_{\Gamma \sim \gP (X \sim P_d, Y \sim P_{\theta})}\mathbb{E}_{(X, Y) \sim \Gamma} \bigl[ c(X,Y)\bigr],
\end{equation}
where $\gP (X \sim P_d, Y \sim P_{\theta})$ is a set of all joint distributions of $\big(P_d; P_{\theta} \big)$; $c: \gX_{\rmO} \times \gX_{\rmO} \mapsto \gR_+$ is any measurable cost function over $\gX_{\rmO}$ (i.e., the product space of the spaces of observed variables) defined as $c(X_\rmO, Y_\rmO) := \sum_{i \in \rmO} c_i(X_i, Y_i)$ where $c_i$ is a measurable cost function over a space of an observed variable.

Since $P_{\theta}$ is intractable due to the latent factor, the formulation in Eq. (\ref{eq:kanto}) cannot be directly optimized. We now propose our solution to this optimization problem (OP). 

Let $P_{\theta} (\mathrm{PA}_{X_i}, U_i)$ denote the joint distribution of $\mathrm{PA}_{X_i}$ and $U_i$ factorized according to the graphical model. Let $\gU_i$ denote the space over random variable $U_i$. The key ingredient of our theoretical development is local backward mapping. For every observed node $i \in \rmO$, we define a stochastic ``backward'' map $\phi_i  : \gX_i \mapsto \Pi_{k \in \mathrm{PA}_{X_i}} \ \gX_k \times \gU_i$ such that $\phi_i \in \mathfrak{C}(X_i)$ where $\mathfrak{C}(X_i)$ is the constraint set given as 
\begin{align*}
    \mathfrak{C}(X_i) := \bigl\{\phi_i : \phi_{i}\# P_d(X_i) = P_{\theta} (\mathrm{PA}_{X_i}, U_i) \bigr\};
\end{align*}
that is, every backward $\phi_{i}\#$ defines a push forward operator such that the samples from $\phi_i(X_i)$ follow the marginal distribution $P_{\theta}(\mathrm{PA}_{X_i}, U_i)$. Let $P_{\phi_i}(\mathrm{PA}_{X_i}, U_i) = \mathbb{E}_{X_i}\left[ \phi_i(\mathrm{PA}_{X_i}, U_i \vert X_i) \right]$ denote the marginal distribution induced by every $\phi_i$. 

\begin{figure*}[ht!]
\centering
     \begin{subfigure}
         \centering\includegraphics[width=0.2\linewidth]{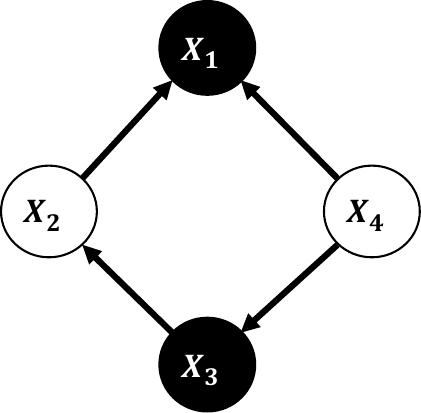}
     \end{subfigure} ~
     \begin{subfigure}
         \centering
         \includegraphics[width=0.35\linewidth]{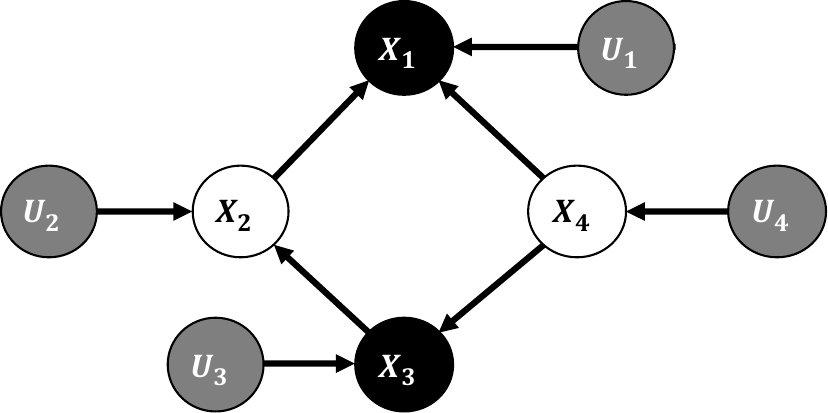}
     \end{subfigure} ~
     \begin{subfigure}
         \centering
         \includegraphics[width=0.35\linewidth]{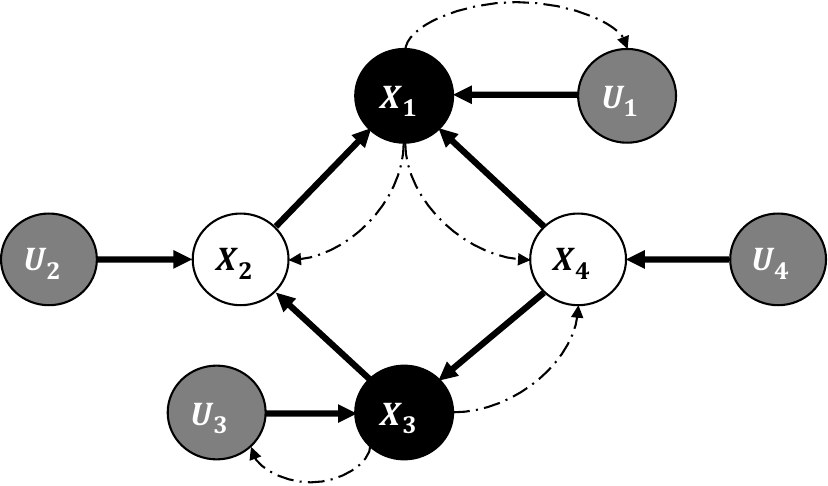}
     \end{subfigure}
      \caption{\textbf{(Left)} A DAG represents a system of $4$ endogenous variables where $X_1, X_3$ are observed (black-shaded) and $X_2, X_4$ are hidden variables (non-shaded). \textbf{(Middle)} The extended DAG includes an additional set of independent exogenous variables $U_1, U_2, U_3, U_4$ (grey-shaded) acting on each endogenous variable. $U_1, U_2, U_3, U_4 \sim P(U)$ where $P(U)$ is a prior product distribution. \textbf{(Right)} Visualization of our backward-forward algorithm, where the dashed arcs represent the backward maps involved in optimization. 
     }
\label{fig:main}   
\end{figure*}

We will show that the OP in (\ref{eq:kanto}) amounts to minimizing the reconstruction error between the observed data and the data generated from $P_{\theta}$. To understand how the reconstruction works, let us examine the right illustration in Figure \ref{fig:main}. With a slight abuse of notations, for every $X_i$, we extend its parent set $\mathrm{PA}_{X_i}$ to include an exogenous variable and possibly some other endogenous variables. Given $X_1$ and $X_3$ as observed nodes, we first sample $X_1 \sim P_d(X_1), X_3 \sim P_d(X_3)$ and then construct backward maps $\phi_1, \phi_3$. The next step is to sample $\textrm{PA}_{X_1} \sim \phi_1( \textrm{PA}_{X_1} \vert X_1)$ and $\textrm{PA}_{X_3} \sim \phi_3( \textrm{PA}_{X_3} \vert X_3)$, where $\textrm{PA}_{X_1} = \{X_2, X_4, U_1\}$ and $\textrm{PA}_{X_3} = \{X_4, U_3\}$, which are plugged back to the model $\psi_{\theta}$ to obtain the reconstructions $\widetilde{X_1} = \psi_{\theta_1}(\textrm{PA}_{X_1})$ and $\widetilde{X_3} = \psi_{\theta_3}(\textrm{PA}_{X_3})$. We wish to learn $\theta$ such that $X_1$ and $X_3$ are reconstructed correctly. For a general graphical model, this optimization objective is formalized as

\begin{theorem}\label{theorem:1} For every $\phi_i$ as defined above and fixed $\psi_{\theta}$, 
\begin{align}\label{eq:primalobj}
& W_c \big(P_d(X_\rmO); P_{\theta}(X_\rmO)\big) = \nonumber  \\ 
& \inf_{\bigl[\phi_i \in \mathfrak{C}(X_i)\bigr]_{i \in \rmO}}\mathbb{E}_{
    X_\rmO \sim P_d ,
    \mathrm{PA}_{X_\rmO} \sim \phi(X_\rmO)} 
    \bigl[ c \bigl(X_\rmO, \psi_{\theta}(\mathrm{PA}_{X_\rmO}) \bigr) \bigr],
\end{align}
where $\mathrm{PA}_{X_\rmO} := \big[[X_{ij}]_{j \in \mathrm{PA}_{X_i}}\big]_{i \in \rmO}$.
\end{theorem}
\paragraph{Proof.} See Appendix \ref{sup:proof}.

By Theorem \ref{theorem:1}, the estimation of OT cost is reduced to finding the optimal conditional $\phi(\mathrm{PA}_{X_i} \vert X_i)$ for every observed node $X_i$ such that the ``backward" marginal $P_{\phi}(\mathrm{PA}_{X_i})$ is identical to the the ``forward" marginal $P_{\theta}(\mathrm{PA}_{X_i})$. While Theorem \ref{theorem:1} set ups a tractable form for our optimization solution, our OP is constrained, where every backward function $\phi_i$ must satisfy its push-forward condition defined by $\mathfrak{C}$. In the above example, the backward maps $\phi_i$ and $\phi_3$ must be constructed such that $\phi_1\#P_{d}(X_1) = P_{\theta}(X_2, X_4, U_1)$ and $\phi_3\#P_{d}(X_3) = P_{\theta}(X_4, U_3)$. We propose to relax the constraints by adding a penalty to the  objective (\ref{eq:primalobj}). 

The \textbf{final optimization objective} is therefore given as
\begin{align}\label{eq:finalobj}
& \inf_{\theta} \ \inf_{\phi} \mathbb{E}_{
    X_\rmO \sim P_d,
    \mathrm{PA}_{X_\rmO} \sim \phi(X_\rmO)
    } 
    \bigl[ c \bigl(X_\rmO, \psi_{\theta}(\mathrm{PA}_{X_\rmO})\bigr)\bigr] \nonumber \\ 
    & \quad \quad \quad \quad + \eta \ D\big( P_{\phi}, P_{\theta} \big) 
\end{align}
where $D$ is any arbitrary divergence measure and $\eta > 0$ is a trade-off hyper-parameter. $D \big( P_{\phi}, P_{\theta} \big)$ is a short-hand for divergence between all pairs of backward and forward marginals over the parents of the observed nodes. 

\paragraph{Remark.}\label{sect:ae_relation} Eq. (\ref{eq:finalobj}) renders an optimization-based approach in which we leverage reparameterization and amortization \citep{amos2022tutorial} for solving it efficiently via stochastic gradient descent. This theoretical result provides our OTP-DAG with two interesting properties: (1) all model parameters are optimized simultaneously within a single framework whether the variables are continuous or discrete, and (2) the computational process can be automated without the need for analytic lower bounds (as in EM and traditional VI), specific graphical structures (as in mean-field VI), or priors over variational distributions on latent variables (as in hierarchical VI). The flexibility our method exhibits is akin to auto-encoding models, and OTP-DAG in fact serves as an extension of WAE for learning general directed graphical models. Our formulation thus inherits a desirable characteristic from WAE, which specifically helps mitigate the posterior collapse issue notoriously occurring to VAE. Appendix \ref{sup:ae_relation} explains this behavior in more detail. Particularly in Section \ref{sect:drepl}, we will empirically show that OTP-DAG effectively alleviates the codebook collapse issue in discrete representation learning. Algorithm \ref{alg:algo} provides the pseudo-code for OTP-DAG learning procedure.

\begin{algorithm}[ht!]
    \caption{OTP-DAG Algorithm}
    \label{alg:algo}
\begin{algorithmic}

\STATE \textbf{Input:} Directed graph $\rmG$ with observed nodes $\rmO$, noise distribution $P(U)$, regularization coefficient $\eta$, reconstruction cost function $c$, and divergence measure $D$.

\STATE \textbf{Output:} Point estimate $\theta$.

\STATE  Initialize a set of deterministic assignments $\psi_{\theta} = \{\psi_{\theta_i}\}_{i\in \rmO}$ where $X_i := \psi_{\theta_i}(\textrm{PA}_{X_i}, U_i)$ and $U_i \sim P(U)$;

\STATE Initialize the stochastic  maps $\phi = \{\phi_i(X_i)\}_{i \in \rmO}$;

\STATE \textbf{while} $(\phi, \theta)$ \textit{not converged} \textbf{do}    
 \STATE \ \ \ \  \textit{for} $i \in \rmO$,
  \begin{itemize}[noitemsep]
      \item Sample batch $X^{B}_{i} = \{ x^{1}_{i}, ..., x^{B}_{i} \}$; 
      \item Sample $\widetilde{\textrm{PA}}_{X^{B}_{i}}$ from $\phi_i(X^{B}_{i})$;
      \item Sampling $U_i$ from the prior $P(U)$;
      \item Evaluate $\widetilde{X}^{B}_{i} = \psi_{\theta_i}(\widetilde{\textrm{PA}}_{X^{B}_{i}}, U_i)$.
  \end{itemize}
  
  \STATE Update $\phi$ and $\theta$ alternately by descending
    $$\frac{1}{B} \sum_{b=1}^{B} \sum_{i\in \rmO} c \big(x_{i}^{b}, \widetilde{x}_{i}^{b} \big) + \eta \ D \big[ P_{\phi_i}(\textrm{PA}_{X_i^B}), P_{\theta}(\textrm{PA}_{X_i^B}) \big]$$
    
\STATE \textbf{end while}
\end{algorithmic}
\end{algorithm}

\section{Applications}
In this section, we illustrate the practical applications of the OTP-DAG algorithm. We consider various directed probabilistic models with different types of latent variables (continuous and discrete) and for different types of data (texts, images, and time series). In all tables, we report the average results over $5$ random initializations and the best/second-best ones are bold/underlined. $\uparrow$, $\downarrow$ indicate higher/lower performance is better, respectively.

\paragraph{Baselines.}
We compare OTP-DAG with two groups of parameter learning methods towards the two extremes: (1) EM and SVI where analytic derivation is required; (2) variational auto-encoding frameworks where black-box optimization is permissible. We leave the discussion of the formulation and technicalities in Appendix \ref{sup:exp}. 

\paragraph{Experimental setup.}  We begin with (1) Latent Dirichlet Allocation \cite{blei2003latent}, a popular task of topic modeling where traditional methods like EM or SVI can solve. We then consider learning a (2) Hidden Markov Model (HMM), which remains fairly challenging, where existing optimization/inference algorithms (e.g., Baum-Welch algorithm) are often too computationally costly to be used in practice. We conclude with a more challenging setting: (3) Discrete Representation Learning (Discrete RepL) that cannot simply be solved by EM or MAP (maximum a posteriori). It in fact invokes deep generative modeling via a pioneering development called Vector Quantization Variational Auto-Encoder \citep[VQ-VAE,][]{van2017neural}. We attempt to apply OTP-DAG for learning discrete representations by grounding it into a parameter learning problem. We note again that for standard (continuous) representation learning, OTP-DAG reduces to WAE \citep{tolstikhin2017wasserstein}, which readers can refer to for extensive empirical evidence. Identifiability of the parameters in latent variable models is of critical concern. In the task of recovering the true parameters, we experiment with the LDA setting and Poisson HMM where the parameters are identifiable up to permutations \citep{teicher1960mixture,wang2019convergence}, and we resolve the ambiguity by sorting out the estimations. 

Figure \ref{fig:experiment} illustrates the empirical DAG structures of $3$ applications. Unlike the standard visualization where the parameters are considered hidden nodes, our graph separates model parameters from latent variables and only illustrates random variables and their dependencies (except the special setting of discrete representation learning). We also omit the exogenous variables associated with the hidden nodes for visibility, since only those acting on the observed nodes are relevant for computation. There is also a noticeable difference between Figures \ref{fig:experiment} and \ref{fig:main}: the empirical version does not require learning the backward maps for the exogenous variables. It is observed across our experiments that sampling the noise from an appropriate prior distribution suffices to yield accurate estimation, which is in fact beneficial in that the training time can be greatly reduced.

\paragraph{Remark.} In the following, we show that in the simulated settings where the models are well-specified, OTP-DAG performs equally well as the baseline methods, while exhibits superior efficiency over EM on such a complex graph as HMM. Furthermore, OTP-DAG is shown to achieve better performance on real-world downstream tasks, which substantiates the robustness of the minimum Wasserstein estimators in practical settings. Finally, throughout the experiments, we also aim to demonstrate the versatility of OTP-DAG where our method can be harnessed for a wide range of purposes in a single learning procedure.

\begin{figure*}[ht!]
     \centering
     \begin{subfigure}
         \centering
         \includegraphics[width=0.25\linewidth]{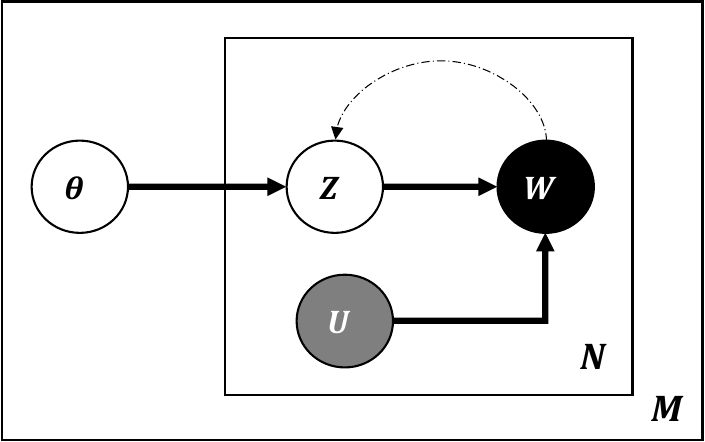}
         \label{fig:plsi}
     \end{subfigure}~
     \hspace{2mm}
     \begin{subfigure}
         \centering
         \includegraphics[width=0.25\linewidth]{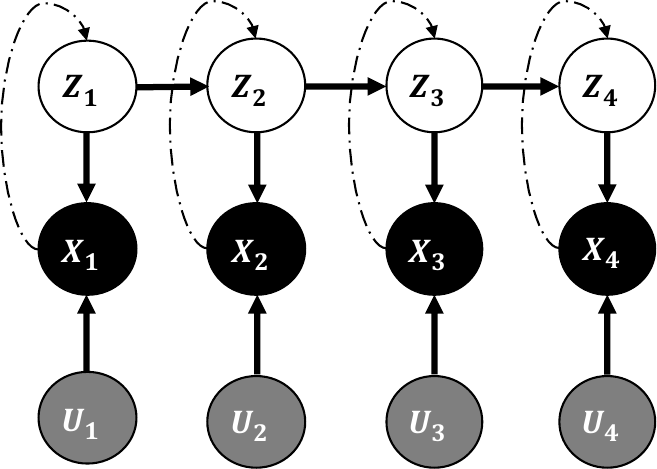}
         \label{fig:hmm}
     \end{subfigure}~
     \hspace{2mm}
     \begin{subfigure}
         \centering
         \includegraphics[width=0.25\linewidth]{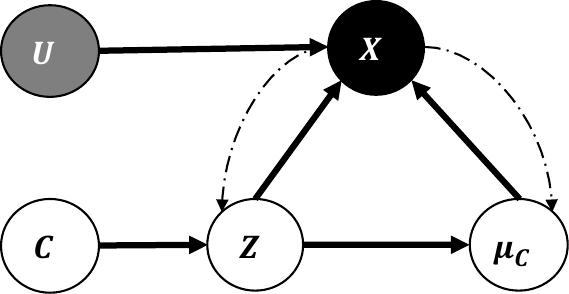}
         \label{fig:vqwae}
     \end{subfigure}
      \caption{Empirical structures of \textbf{(left)} latent Dirichlet allocation model (in plate notation), \textbf{(middle)} standard hidden Markov model, and \textbf{(right)} discrete representation learning. 
     }
\label{fig:experiment}   
\end{figure*}

\subsection{Latent Dirichlet Allocation}
Let us consider a corpus $\gD$ of $M$ independent documents where each document is a sequence of $N$ words denoted by $W_{1:N} = (W_1, W_2, \cdots, W_N)$. Documents are represented as random mixtures over $K$ latent topics, each of which is characterized by a distribution over words. Let $V$ be the size of a vocabulary indexed by $\{1, \cdots, V\}$. Latent Dirichlet Allocation (LDA) \citep{blei2003latent} dictates the following generative process for every document in the corpus:

\begin{enumerate}[noitemsep]
    \item Sample $\theta \sim \textrm{Dir}(\alpha)$ with $\alpha < 1$,
    \item Sample $\gamma_k \sim \textrm{Dir}(\beta)$ where $k \in \{1, \cdots, K\}$,
    \item For each of the word positions $n \in \{1, \cdots, N\}$,
    \begin{itemize}[noitemsep]
        \item Sample a topic $z_n \sim \textrm{Multi-nominal}(\theta)$,
        \item Sample a word $w_n \sim \textrm{Multi-nominal}(\gamma_{kn})$,
    \end{itemize}
\end{enumerate}

where $\textrm{Dir}(.)$ is a Dirichlet distribution. $\theta$ is a $K-$dimensional vector that lies in the $(K-1)-$simplex and $\gamma_k$ is a $V-$dimensional vector represents the word distribution corresponding to topic $k$. In the standard model, $\alpha, \beta, K$ are hyper-parameters and $\theta,\gamma$ are learnable parameters.  Throughout the experiments, the number of topics $K$ is assumed known and fixed. 

\paragraph{Parameter estimation.} To test whether OTP-DAG can recover the true parameters, we generate synthetic data in the setting: the word probabilities are parameterized by a $K \times V$ matrix $\gamma$ where $\gamma_{kn} := P(W_n = 1 \vert Z_n =1)$; $\gamma$ is now a fixed quantity to be estimated. We set $\alpha = 1/K$ uniformly and generate small datasets for different number of topics $K$ and sample size $N$. Following \cite{griffiths2004finding}, for every topic $k$, the word distribution $\gamma_k$ can be represented as a square grid where each cell, corresponding to a word, is assigned an integer value of either $0$ and $1$, indicating whether a certain word is allocated to the $k^{th}$ topic or not. As a result, each topic is associated with a specific pattern. For simplicity, we represent topics using horizontal or vertical patterns (See Figure \ref{fig:topic10}). According to the above generative model, we sample data w.r.t $3$ sets of configuration triplets $\{K,M,N\}$. We compare OTP-DAG with Batch EM and SVI and Prod LDA - a variational auto-encoding topic model \citep{srivastava2017autoencoding}. 

\begin{table*}[h!]
\caption{Fidelity of estimates of the topic-word distribution $\gamma$ across $3$ settings. Fidelity is measured by KL divergence, Hellinger (HL) \citep{hellinger1909neue} and Wasserstein distance with the true $\gamma$.}
\vskip 0.15in
\centering
\resizebox{0.85\linewidth}{!}{
\begin{tabular}{l l l l c c c r}
\toprule
Metric $\downarrow$ & $K$ & $M$ & $N$ & \textbf{OTP-DAG} (Ours) & \textbf{Batch EM} & \textbf{SVI} & \textbf{Prod LDA} \\
\midrule
HL  & 10  & 1,000  & 100 & \textbf{2.327 $\pm$ 0.009} & 2.807 $\pm$ 0.189          & 2.712 $\pm$ 0.087          & \underline{2.353 $\pm$ 0.012} \\
KL & 10  & 1,000  & 100 & 1.701 $\pm$ 0.005          & 1.634 $\pm$ 0.022 & \textbf{1.602 $\pm$ 0.014} & \underline{1.627 $\pm$ 0.027}          \\
WS & 10  & 1,000  & 100 & \textbf{0.027 $\pm$ 0.004} & 0.058 $\pm$ 0.000          & 0.059 $\pm$ 0.000          & \underline{0.052 $\pm$ 0.001} \\

\midrule
HL & 20  & 5,000  & 200 & \underline{3.800 $\pm$ 0.058} & 4.256 $\pm$ 0.084          & 4.259 $\pm$ 0.096          & \textbf{3.700 $\pm$ 0.012} \\
KL & 20  & 5,000  & 200 & 2.652 $\pm$ 0.080          & \textbf{2.304 $\pm$ 0.004} & \underline{2.305 $\pm$ 0.003} & 2.316 $\pm$ 0.026          \\
WS & 20  & 5,000  & 200 & \textbf{0.010 $\pm$ 0.001} & 0.022 $\pm$ 0.000          & 0.022 $\pm$ 0.001          & \underline{0.018 $\pm$ 0.000}          \\

\midrule
HL & 30  & 10,000 & 300 & \underline{4.740 $\pm$ 0.029} & 5.262 $\pm$ 0.077          & 5.245 $\pm$ 0.035          & \textbf{4.723 $\pm$ 0.017} \\
KL & 30  & 10,000 & 300 & 2.959 $\pm$ 0.015          & \textbf{2.708 $\pm$ 0.002} & \underline{2.709 $\pm$ 0.001} & 2.746 $\pm$ 0.034          \\
WS & 30  & 10,000 & 300 & \textbf{0.005 $\pm$ 0.001} & 0.012 $\pm$ 0.000          & 0.012 $\pm$ 0.000          & \underline{0.009 $\pm$ 0.000}         
\\
\bottomrule
\end{tabular}
}
\label{tab:lda_param}
\end{table*}

Table \ref{tab:lda_param} reports the fidelity of the estimation of $\gamma$. OTP-DAG consistently achieves high-quality estimates by both Hellinger and Wasserstein distances. It is not surprising that the baselines are superior by the KL metric, as it is what they implicitly minimize. While it is inconclusive from the numerical estimations, the qualitative results complete the story. Figure \ref{fig:topic10} illustrates the distributions of individual words to the topics from each method after 300 training epochs. OTP-DAG successfully recovers the true patterns and as well as EM and SVI, while Prod LDA mis-detects several patterns, despite the competitive numerical results. More qualitative examples for the other settings are presented in Figures \ref{fig:topic20} and \ref{fig:topic30} where OTP-DAG is shown to recover almost all true patterns.

\begin{figure}
    \centering    \includegraphics[width=\linewidth]{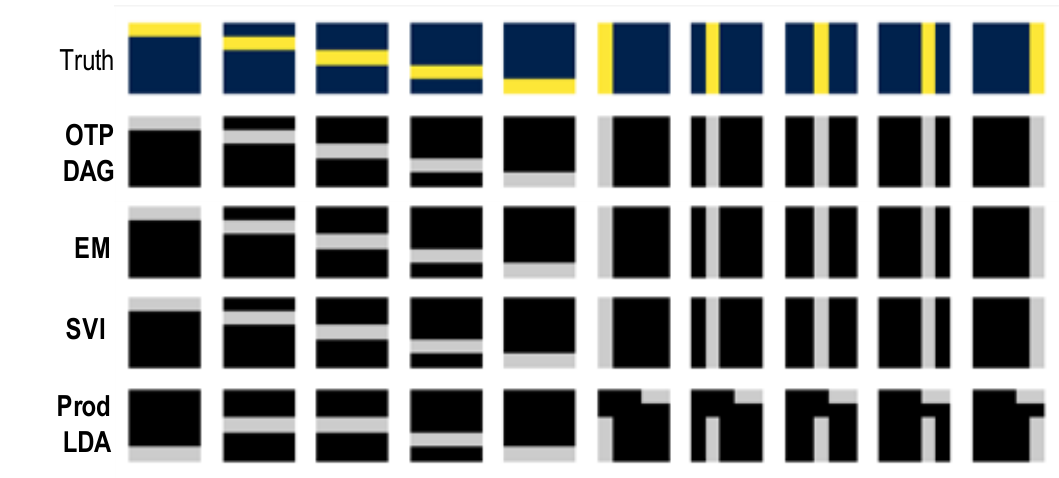}
    \caption{Topic-word distributions inferred by each method from the 1st set of synthetic data after 300 training epochs.}
    \label{fig:topic10}
\end{figure}

\paragraph{Topic Inference.} We now demonstrate the effectiveness of OTP-DAG on downstream applications\footnote{\url{https://github.com/MIND-Lab/OCTIS}. We use OCTIS to standardize evaluation for all models on the topic inference task. Note that the computation of topic coherence score in OCTIS is different than in \citet{srivastava2017autoencoding}.}. We here use OTP-DAG to infer the topics of $3$ real-world datasets: 20 News Group\footnote{\url{http://qwone.com/~jason/20Newsgroups/}.}, BBC News \citep{greene06icml} and DBLP\footnote{\url{https://github.com/shiruipan/TriDNR/}.}. We revert to the original generative process where the topic-word distributions follows a Dirichlet distribution parameterized by the concentration parameters $\beta$, instead of having $\gamma$ as a fixed quantity. $\beta$ is now initialized as a matrix of real values $\big( \beta \in \mathbb{R}^{K \times V}\big)$ representing the log concentration values.  

For every topic $k$, we select top $10$ most related words according to $\gamma_k$ to represent it. Table \ref{tab:lda_quanti} reports the quality of the inferred topics, which is evaluated via the diversity and coherence of the selected words. Diversity refers to the proportion of unique words, whereas Coherence is measured with normalized pointwise mutual information \citep{aletras2013evaluating}, reflecting the extent to which the words in a topic are associated with a common theme. There exists a trade-off between Diversity and Coherence: words that are excessively diverse greatly reduce coherence, while a set of many duplicated words yields higher coherence yet harms diversity. A well-performing topic model would strike a good balance between these metrics~\cite{zhao2021topic}. If we consider two metrics comprehensively, our method consistently achieves better performance across different settings. Qualitative results of the inferred topics can be found in Table \ref{tab:lda_quali}.

\begin{table}[h!]
    \centering
\caption{Coherence and Diversity of the inferred topics for the $3$ real-world datasets ($K = 10$).}
\vskip 0.15in
\resizebox{\linewidth}{!}{
\begin{tabular}{l c c c r}
\toprule
Metric (\%) $\uparrow$ & \textbf{OTP-DAG} (Ours)   & \textbf{Batch EM} & \textbf{SVI}              & \textbf{Prod LDA} \\
\midrule
\multicolumn{5}{c}{20 News Group} \\
\midrule
Coherence                & \textbf{10.45 $\pm$ 0.56} & \underline{6.71 $\pm$ 0.16} & 5.90 $\pm$ 0.51           & 4.78 $\pm$ 2.64           \\
Diversity                & \underline{92.00 $\pm$ 2.65} & 72.33 $\pm$ 1.15      & 85.33 $\pm$ 5.51          & \textbf{92.67 $\pm$ 4.51}    \\
\midrule
\multicolumn{5}{c}{BBC News} \\
\midrule
Coherence                & \textbf{9.12 $\pm$ 0.81}        & \underline{8.67 $\pm$ 0.62} & 7.84 $\pm$ 0.49           & 2.17 $\pm$ 2.36                 \\
Diversity                & \underline{87.67 $\pm$ 2.65}                & 86.00 $\pm$ 1.00      & \textbf{92.33 $\pm$ 2.31} & \underline{87.67 $\pm$ 3.79}          \\
\midrule
\multicolumn{5}{c}{DBLP} \\
\midrule
Coherence                & \textbf{7.66 $\pm$ 0.44}  & \underline{4.52 $\pm$ 0.53} & 1.47 $\pm$ 0.39           & 2.91 $\pm$ 1.70           \\
Diversity                & \underline{97.33 $\pm$ 1.53}    & 81.33 $\pm$ 1.15      & 92.67 $\pm$ 2.52          & \textbf{98.67 $\pm$ 1.53} 
\\
\bottomrule
\end{tabular}
}
\label{tab:lda_quanti}
\end{table}

\subsection{Hidden Markov Models}\label{sect:hmm}
This application deals with time-series data following a \textbf{Poisson hidden Markov model}. Given a time series of $T$ steps, the task is to segment the data stream into 4 different states, each of which follows a Poisson distribution with rate $\lambda_k$ sampled from a Uniform hyper-prior. The distributions and the observation at each step $t$ are given as 
\begin{equation*}
    P(X_t \vert Z_t = k) = \mathrm{Poi}(X_t \vert \lambda_k), \quad \text{for} \ k=1, \cdots, 4,
\end{equation*}
where $\lambda_1 \sim U(10,20)$, $\lambda_2 \sim U(30,40)$, $\lambda_3 \sim U(50,60)$ and $\lambda_4 \sim U(80,90)$.
We further impose a uniform prior over the initial state. The Markov chain stays in the current state with probability $p$ and otherwise transitions to one of the other three states uniformly at random. The transition distribution is given as 
\begin{align*}
    & z_1 \sim \mathrm{Cat} \bigg( \bigg\{ \frac{1}{4}, \frac{1}{4}, \frac{1}{4}, \frac{1}{4} \bigg\} \bigg), \\
    & z_t \vert z_{t-1} \sim \mathrm{Cat} \bigg( 
    \left\{\begin{array}{lr}
        \pi & \text{if } Z_t = Z_{t-1}\\
        \frac{1-\pi}{4-1} & \text{otherwise } 
        \end{array}
        \right\}\bigg)
\end{align*}

We randomly generate $200$ datasets of $50,000$ observations each. For each dataset, we train the models for $50$ epochs with learning rate of $0.05$ at $5$ different initializations. We would like to learn the concentration parameters $\lambda_{1:4}$ through which segmentation can be realized, assuming that the number of states is known. The other experimental configuration is reported in Appendix \ref{sup:hmm_exp}. 

Table \ref{tab:poisonhmm} reports mean error of the estimates of the parameters along with runtime of OTP-DAG and EM. As the absolute values can be misleading, we report the errors in relative terms, where we apply min-max normalization to scale the $\lambda$ values to $[0,1]$.  Figure \ref{fig:estimate_dist} additionally visualizes the distribution of the estimations from OTP-DAG and EM to show the alignment with the generative uniform distributions.

\begin{table}[hbt!]
\caption{Estimates of the concentration parameters $\lambda_{1:4}$ of the Poisson HMM, measured by mean absolute error with the true values.}
\vskip 0.15in
\centering
     \resizebox{\linewidth}{!}{
    \begin{tabular}{l|c|c}
    \toprule
     Method & \textbf{OTP-DAG} (Ours) & \textbf{EM}\\
    \midrule
     $\lambda_1$  & 0.040 $\pm$ 0.129 & \textbf{0.022 $\pm$ 0.042} \\ 

     $\lambda_2$  & \textbf{0.079 $\pm$ 0.088} & 0.088 $\pm$ 0.105 \\ 

     $\lambda_3$  & \textbf{0.148 $\pm$ 0.119} & 0.166 $\pm$ 0.171 \\ 

     $\lambda_4$  & \textbf{0.084 $\pm$ 0.099} & 0.101 $\pm$ 0.008 \\ 

     Runtime ($50$ steps) & $\approx \mathbf{3}$ \textbf{mins} & $\approx 20$ mins \\
    \bottomrule
    \end{tabular}
     }    
    \label{tab:poisonhmm}
\end{table}



\subsection{Learning Discrete Representations}\label{sect:drepl}
\begin{table*}[h!]
    \centering
    \caption{Quality of the image reconstructions from the vector quantized models ($K = 512$).}
    \vskip 0.15in
\resizebox{0.85\linewidth}{!}{
    \begin{tabular}{lrcrrrrr}
        \toprule 
Dataset & Method & Latent Size & SSIM $\uparrow$ & PSNR $\uparrow$ & LPIPS $\downarrow$ & rFID $\downarrow$ & Perplexity $\uparrow$
\tabularnewline 
\midrule 
CIFAR10 & \textbf{VQ-VAE}  & 8 $\times$ 8 & 0.70 & 23.14 & 0.35 & 77.3 & 69.8  \tabularnewline
 & \textbf{OTP-DAG} (Ours) & 8 $\times$ 8 & \textbf{0.80} & \textbf{25.40} & \textbf{0.23} & \textbf{56.5} & \textbf{498.6} \tabularnewline
 \midrule 
         
MNIST & \textbf{VQ-VAE}  & 8 $\times$ 8 & \textbf{0.98} & 33.37 & 0.02  & 4.8 & 47.2 \tabularnewline
& \textbf{OTP-DAG} (Ours)  & 8 $\times$ 8 & \textbf{0.98} & \textbf{33.62} & \textbf{0.01} & \textbf{3.3} & \textbf{474.6}
\tabularnewline
 \midrule 
SVHN & \textbf{VQ-VAE} & 8 $\times$ 8 & 0.88 & 26.94 & 0.17 & 38.5 &  114.6 \tabularnewline
& \textbf{OTP-DAG} (Ours) & 8 $\times$ 8 & \textbf{0.94} & \textbf{32.56} & \textbf{0.08} & \textbf{25.2} & \textbf{462.8} \tabularnewline
\midrule
CELEBA & \textbf{VQ-VAE} & 16 $\times$ 16 & 0.82 & 27.48 & 0.19 &  19.4 & 48.9 \tabularnewline
& \textbf{OTP-DAG} (Ours)  & 16 $\times$ 16  & \textbf{0.88} & \textbf{29.77} & \textbf{0.11} & \textbf{13.1} & \textbf{487.5} \tabularnewline
\bottomrule
\end{tabular}
}
\label{tab:drl}
\end{table*}


Learning latent discrete representations of data is an important problem, which can be useful for planning and symbolic reasoning tasks. Viewing discrete representation learning as a parameter learning problem, we endow it with a probabilistic generative process as illustrated in Figure \ref{fig:vqwae}. 
The problem deals with a latent space $\gC \in \mathbb{R}^{K \times D}$ composed of $K$ discrete latent sub-spaces of $D$ dimensionality. The probability a data point belongs to a discrete sub-space $c \in \{1, \cdots, K\} $ follows a $K-$way categorical distribution $\pi = [\pi_1, \cdots, \pi_K]$. In the language of VQ-VAE, each $c$ is referred to as a \textit{codeword} and the set of codewords is called a \textit{codebook}. Let $Z \in \mathbb{R}^D$ denote the latent variable in a sub-space. On each sub-space, we impose a Gaussian distribution parameterized by $\mu_c, \Sigma_c$ where $\Sigma_c$ is diagonal. 
The generative process is as follows:
\begin{enumerate} [noitemsep]
    \item Sample $c \sim \textrm{Cat}(\pi)$ and $z \sim \gN(\mu_c, \Sigma_c)$
    \item Quantize $\mu_c = Q(z)$,
    \item Generate $x = \psi_{\theta}(z, \mu_c)$,
\end{enumerate}
where $\psi$ is a highly non-convex function with unknown parameters $\theta$ and often parameterized by a deep neural network. $Q$ refers to the quantization of $z$ to $\mu_c$ defined as $\mu_c = Q(z)$ where $c = \mathrm{arg min}_{c} \ d_z \big(z; \mu_c \big)$ and $d_z = \sqrt{(z - \mu_c)^{T} \Sigma_{c}^{-1} (z - \mu_c)}$ is the Mahalanobis distance. The goal is to learn the parameter set $\{\pi, \mu, \Sigma, \theta\}$ with $\mu = [\mu_k]_{k=1}^K, \Sigma = [\Sigma_k]_{k=1}^K$ such that the learned representation captures the key properties of the data. Following VQ-VAE, our practical implementation considers $Z$ as an $M-$component latent embedding. 

We experiment with images in this application and compare OTP-DAG with VQ-VAE on CIFAR10\footnote{\url{https://www.cs.toronto.edu/~kriz/cifar.html}.}, MNIST \citep{lecun1998gradient}, SVHN \citep{netzer2011reading} and CELEBA datasets \cite{liu2015deep}. Since the true parameters are unknown, we assess how well the latent space characterizes the input data through the quality of the reconstruction of the original images. Table \ref{tab:drl} reports our superior performance in preserving high-quality information of the input images. VQ-VAE suffers from poorer performance mainly due to \textit{codebook collapse} \citep{yu2021vector} where most of latent vectors are quantized to limited discrete codewords. Meanwhile, our framework allows to control the number of latent representations, ensuring all codewords are utilized. In Appendix \ref{sup:repl_exp}, we detail the formulation of our method and provide qualitative examples. We also showcase therein our competitive performance against a recent advance called SQ-VAE \citep{takida2022sq} without introducing any additional complexity.

\section{Discussion and Conclusion}\label{sect:discussion}
\paragraph{Discussion.} The key message across our experiments is that OTP-DAG is a scalable and versatile framework readily applicable to learning any directed graphs with latent variables. Similar to amortized VI, on one hand, our method employs amortized optimization and assumes one can sample from the priors or more generally, the model marginals over latent parents. OTP-DAG requires continuous relaxation through reparameterization of the underlying model distribution to ensure the gradients can be back-propagated effectively. Note that this specification is not unique to OTP-DAG: VAE also relies on the reparameterization trick to compute the gradients w.r.t the variational parameters. For discrete distributions and for non-reparameterizable continuous ones (e.g., Gamma distribution),  the reparameterization trick cannot be easily applied. To this end, a proposal on \textit{Generalized Reparameterization Gradient} \citep{ruiz2016generalized} can be a viable solution. 

On the other hand, different from VI, our global OT cost minimization is achieved by characterizing local densities through the backward maps from the observed nodes to their parents.  This localization strategy makes it easier to find a good approximation compared to VI, where the variational distribution is defined over all hidden variables and should ideally characterize the entire global dependencies in the graph.  To model the backward distributions, we utilize the expressive power of deep neural networks. Based on the universal approximation theorem \citep{hornik1989multilayer}, the gap between the backward and forward marginals can be assumed to be smaller than an arbitrary constant $\epsilon$ given enough data, network complexity, and training time. 

\paragraph{Limitations.} Theoretically, our algorithm can scale up to more complex graphs since we make no assumptions about the graphical structure. Our algorithm remains applicable to different graph sizes, where the computation is localized to the dependencies between an observed node and its (direct) parents. However, larger graphs indeed induce more operations where ancestral sampling to evaluate the model marginals over the related parent nodes can be computational expensive. The increased complexity has little impact on our evaluation of reconstruction loss,
which only involves forward operations. However, an immediate trade-off arises as it introduces additional computational
complexity to the optimization of the divergence measures. Fortunately our framework allows $D$ to be chosen flexibly depending on applications. Here we
analyze some promising candidates. A popular option is to choose $D$ as the Jensen–Shannon divergence and estimate it with GAN-based training \citep{goodfellow2020generative}. However, this choice is clearly inappropriate as it necessitates training additional discriminators, not to mention that GANs are known for their instability. Wasserstein (WS) distance and maximum mean discrepancy (MMD) are two other candidates that can be estimated with empirical samples. While the exact computation of WS has high complexity in high-dimensional space, efficient and high-quality approximations exist such as Sinkhorn divergences \citep{cuturi2013lightspeed} or Sliced WS distance \citep{bonneel2015sliced}. MMD is also practically viable, whose sample complexity does not depend on the dimension. Furthermore, there is a kernel-based closed
form to compute an unbiased estimator with reasonable choices of the kernel \citep{gretton2012kernel}.

\paragraph{Future research.}
The proposed algorithm lays the cornerstone for an exciting paradigm shift in the realm of graphical learning and inference. Looking ahead, this fresh perspective unlocks a wealth of promising avenues for future application of OTP-DAG to large-scale inference problems or other learning tasks such as for undirected graphical models, or structural learning where edge existence and directionality can be parameterized as part of the model parameters.


\section*{Acknowledgments}
Trung Le and Dinh Phung were supported by ARC DP23 grant DP230101176 and by the Air Force Office of Scientific Research under award number FA2386-23-1-4044. This does not imply endorsement by the funding agency of the research findings or conclusions. Any errors or misinterpretations in this paper are the sole responsibility of the authors.

\section*{Impact Statement}
This work introduces an application of machine learning to effectively address a class of statistical estimation problems in a scalable manner. While we are currently unaware of any potential negative societal impacts of our work, machine learning frequently yields unintended consequences in various domains, necessitating thorough consideration of societal advantages and drawbacks when implementing the proposed method in real-world scenarios.

\bibliographystyle{icml2024}
\bibliography{ref}
\newpage


\appendix
\onecolumn
\section{Proof}\label{sup:proof}
We now present the proof of Theorem \ref{theorem:1} which is the key theorem in our paper.

\textbf{Theorem \ref{theorem:1}.} \textit{For every $\phi_i$ as defined above and fixed $\psi_{\theta}$}, 
\begin{equation*}\label{eq:sup_main}  
W_c \big(P_d(X_\rmO); P_{\theta}(X_\rmO)\big) = 
\underset{\bigl[\phi_i \in \mathfrak{C}(X_i)\bigr]_{i \in \rmO}}{\mathrm{inf}} \ \mathbb{E}_{
    X_\rmO \sim P_d (X_\rmO),
    \mathrm{PA}_{X_\rmO} \sim \phi(X_\rmO)
    } 
    \bigl[ c \bigl(X_\rmO, \psi_{\theta}(\mathrm{PA}_{X_\rmO}) \bigr) \bigr], 
\end{equation*}
where $\mathrm{PA}_{X_\rmO} := \big[[X_{ij}]_{j \in \mathrm{PA}_{X_i}}\big]_{i \in \rmO}$.

\begin{proof}
    Let $\Gamma \in \mathcal{P}(P_d(X_{\mathbf{O}}),P_\theta(X_{\mathbf{O}}))$ be the optimal joint distribution over $P_d(X_{\mathbf{O}})$ and $P_\theta(X_{\mathbf{O}})$ of the corresponding Wasserstein distance. We consider three distributions: $P_d(X_{\mathbf{O}})$ over $A = \prod_{i\in\mathbf{O}}\mathcal{X}_{i}$, $P_\theta(X_{\mathbf{O}}))$ over $C = \prod_{i\in\mathbf{O}}\mathcal{X}_{i}$, and $ P_\theta(\textrm{PA}_{X_\mathbf{O}})= P_\theta([\textrm{PA}_{X_i}]_{i\in \mathbf{O}})$ over $B = \prod_{i\in\mathbf{O}}\prod_{k\in \textrm{PA}_{X_{i}}}\mathcal{X}_{k}$. Here we note that the last distribution $P_\theta(\textrm{PA}_{X_\mathbf{O}})=P_\theta([\textrm{PA}_{X_i}]_{i\in \mathbf{O}})$ is the model distribution over the parent nodes of the observed nodes. 

    It is evident that $\Gamma \in \mathcal{P}(P_d(X_{\mathbf{O}}),P_\theta(X_{\mathbf{O}}))$ is a joint distribution over $P_d(X_{\mathbf{O}})$ and$P_\theta(X_{\mathbf{O}})$; let $\beta=(id, \psi_\theta)\# P_\theta([\textrm{PA}_{X_i}]_{i\in \mathbf{O}})$ be a deterministic coupling or joint distribution over $P_\theta([\textrm{PA}_{X_i}]_{i\in \mathbf{O}})$ and $P_\theta(X_{\mathbf{O}})$. Using the gluing lemma (see Lemma 5.5 in \cite{santambrogio2015optimal}), there exists a joint distribution $\alpha$ over $A \times B \times C$ such that $\alpha_{AC} = (\pi_A, \pi_C)\#\alpha = \Gamma$ and $\alpha_{BC} = (\pi_B, \pi_C)\#\alpha = \beta$ where $\pi$ is the projection operation. Let us denote $\gamma = (\pi_A, \pi_B)\#\alpha$ as a joint distribution over $P_d(X_{\mathbf{O}})$ and  $P_\theta([\textrm{PA}_{X_i}]_{i\in \mathbf{O}})$. 

    Given $i \in \mathbf{O}$, we denote $\gamma_i$ as the projection of $\gamma$ over $\mathcal{X}_i$ and $\prod_{k \in \textrm{PA}_{X_i}} \mathcal{X}_k$. We further denote $\phi_i(X_i) = \gamma_i(\cdot \mid X_i)$ as a stochastic map from $\mathcal{X}_i$ to $\prod_{k \in \textrm{PA}_{X_i}} \mathcal{X}_k$. It is worth noting that because $\gamma_i$ is a joint distribution over $P_d(X_i)$ and $P_\theta(\textrm{PA}_{X_i})$, $\phi_i \in \mathfrak{C}(X_i)$.  

    \begin{align}
W_{c}\left(P_{d}\left(X_{\mathbf{O}}\right),P_{\theta}\left(X_{\mathbf{O}}\right)\right) & =\mathbb{E}_{\left(X_{\mathbf{O}},\tilde{X}_{\mathbf{O}}\right)\sim\Gamma}\left[c\left(X_{\mathbf{O}},\tilde{X}_{\mathbf{O}}\right)\right]=\mathbb{E}_{\left(X_{\mathbf{O}},\textrm{PA}_{X_{\mathbf{O}}},\tilde{X}_{\mathbf{O}}\right)\sim\alpha}\left[c\left(X_{\mathbf{O}},\tilde{X}_{\mathbf{O}}\right)\right]\nonumber \\
= & \mathbb{E}_{X_{\mathbf{O}}\sim P_{d},\left[\textrm{PA}_{X_{i}}\sim\gamma_i\left(\cdot\mid X_{i}\right)\right]_{i\in\mathbf{O}},\tilde{X}_{\mathbf{O}}\sim\alpha_{BC}\left(\cdot\mid \textrm{PA}_{X_{\mathbf{o}}}\right)}\left[c\left(X_{\mathbf{O}},\tilde{X}_{\mathbf{O}}\right)\right]\nonumber \\
\overset{(1)}{=} & \mathbb{E}_{X_{\mathbf{O}}\sim P_{d},\left[\textrm{PA}_{X_{i}}=\phi_{i}\left(X_{i}\right)\right]_{i\in\mathbf{O}},\tilde{X}_{\mathbf{O}}=\psi_{\theta}\left(\textrm{PA}_{X_{\mathbf{o}}}\right)}\left[c\left(X_{\mathbf{O}},\tilde{X}_{\mathbf{O}}\right)\right]\nonumber \\
= & \mathbb{E}_{X_{\mathbf{O}}\sim P_{d},\textrm{PA}_{X_{\mathbf{O}}}=\phi\left(X_{\mathbf{O}}\right),\tilde{X}_{\mathbf{O}}=\psi_{\theta}\left(\textrm{PA}_{X_{\mathbf{O}}}\right)}\left[c\left(X_{\mathbf{O}},\tilde{X}_{\mathbf{O}}\right)\right]\nonumber \\
\overset{(2)}{=} & \mathbb{E}_{X_{\mathbf{O}}\sim P_{d},\textrm{PA}_{X_{\mathbf{O}}}=\phi\left(X_{\mathbf{O}}\right)}\left[c\left(X_{\mathbf{O}},\psi_{\theta}\left(\textrm{PA}_{X_{\mathbf{O}}}\right)\right)\right]\nonumber \\
\geq & \inf_{\left[\phi_{i}\in\mathfrak{C}(X_{i})\right]_{i\in\mathbf{O}}}\mathbb{E}_{X_{\mathbf{O}}\sim P_{d},\textrm{PA}_{X_{\mathbf{O}}}=\phi\left(X_{\mathbf{O}}\right)}\left[c\left(X_{\mathbf{O}},\psi_{\theta}\left(\textrm{PA}_{X_{\mathbf{O}}}\right)\right)\right].\label{eq:1_side}
\end{align}
Here we note that we have $\overset{(1)}{=}$ because $\alpha_{BC}$ is a deterministic coupling and we have $\overset{(2)}{=}$ because the expectation is preserved through a deterministic push-forward map. 

Let $[\phi_i \in \mathfrak{C}(X_i)]_{i \in \mathbf{O}}$ be the optimal backward maps of the optimization problem (OP) in (\ref{eq:sup_main}). We define the joint distribution $\gamma$ over $P_{d}\left(X_{\mathbf{O}}\right)$ and $P_\theta(\textrm{PA}_{X_\mathbf{O}})=P_\theta([\textrm{PA}_{X_i}]_{i\in \mathbf{O}})$ as follows. We first sample $X_\mathbf{O} \sim P_d(X_\mathbf{O})$ and for each $i \in \mathbf{O}$, we sample $\textrm{PA}_{X_i} \sim \phi_i(X_i)$, and finally gather $(X_\mathbf{O}, \textrm{PA}_{X_\mathbf{O}}) \sim \gamma$ where $\textrm{PA}_{X_\mathbf{O}} = [\textrm{PA}_{X_i}]_{i \in \mathbf{O}}$. Consider the joint distribution $\gamma$ over $P_{d}\left(X_{\mathbf{O}}\right)$ , $P_\theta(\textrm{PA}_{X_\mathbf{O}})=P_\theta([\textrm{PA}_{X_i}]_{i\in \mathbf{O}})$ and the deterministic coupling or joint distribution $\beta=(id, \psi_\theta)\# P_\theta([\textrm{PA}_{X_i}]_{i\in \mathbf{O}})$  over $P_\theta([\textrm{PA}_{X_i}]_{i\in \mathbf{O}})$ and $P_\theta(X_{\mathbf{O}})$, the gluing lemma indicates the existence of the joint distribution $\alpha$ over $A \times C \times B$ such that $\alpha_{AB} = (\pi_A, \pi_B)\#\alpha = \gamma$ and $\alpha_{BC} = (\pi_B, \pi_C)\#\alpha = \beta$. We further denote $\Gamma = \alpha_{AC} = (\pi_A, \pi_C)\#\alpha$ which is a joint distribution over $P_d(X_{\mathbf{O}})$ and $P_\theta(X_{\mathbf{O}})$. It follows that
\begin{align}
 & \inf_{\left[\phi_{i}\in\mathfrak{C}(X_{i})\right]_{i\in\mathbf{O}}}\mathbb{E}_{X_{\mathbf{O}}\sim P_{d},\textrm{PA}_{X_{\mathbf{O}}}=\phi\left(X_{\mathbf{O}}\right)}\left[c\left(X_{\mathbf{O}},\psi_{\theta}\left(\textrm{PA}_{X_{\mathbf{O}}}\right)\right)\right]\nonumber \\
= & \mathbb{E}_{X_{\mathbf{O}}\sim P_{d},\textrm{PA}_{X_{\mathbf{O}}}=\phi\left(X_{\mathbf{O}}\right)}\left[c\left(X_{\mathbf{O}},\psi_{\theta}\left(\textrm{PA}_{X_{\mathbf{O}}}\right)\right)\right]\nonumber \\
\overset{(1)}{=} & \mathbb{E}_{X_{\mathbf{O}}\sim P_{d},\textrm{PA}_{X_{\mathbf{O}}}\sim\phi\left(X_{\mathbf{O}}\right),\tilde{X}_{\mathbf{O}}=\psi_{\theta}\left(\textrm{PA}_{X_{\mathbf{O}}}\right)}\left[c\left(X_{\mathbf{O}},\tilde{X}_{\mathbf{O}}\right)\right]\nonumber \\
= & \mathbb{E}_{X_{\mathbf{O}}\sim P_{d},\textrm{PA}_{X_{\mathbf{O}}}\sim\gamma\left(\cdot\mid X_{\mathbf{O}}\right),\tilde{X}_{\mathbf{O}}\sim\alpha_{BC}\left(\cdot\mid \textrm{PA}_{X_{\mathbf{o}}}\right)}\left[c\left(X_{\mathbf{O}},\tilde{X}_{\mathbf{O}}\right)\right]\nonumber \\
= & \mathbb{E}_{\left(X_{\mathbf{O}},\textrm{PA}_{X_{\mathbf{O}}},\tilde{X}_{\mathbf{O}}\right)\sim\alpha}\left[c\left(X_{\mathbf{O}},\tilde{X}_{\mathbf{O}}\right)\right]\nonumber \\
= & \mathbb{E}_{\left(X_{\mathbf{O}},\tilde{X}_{\mathbf{O}}\right)\sim\Gamma}\left[c\left(X_{\mathbf{O}},\tilde{X}_{\mathbf{O}}\right)\right]\geq W_{c}\left(P_{d}\left(X_{\mathbf{O}}\right),P_{\theta}\left(X_{\mathbf{O}}\right)\right).\label{eq:2_side}
\end{align}
Here we note that we have $\overset{(1)}{=}$ because the expectation is preserved through a deterministic push-forward map. 

Finally, combining (\ref{eq:1_side}) and (\ref{eq:2_side}), we reach the conclusion.
\end{proof}

It is worth noting that according to Theorem \ref{theorem:1}, we need to solve the following OP: 
\begin{equation}
\underset{\bigl[\phi_i \in \mathfrak{C}(X_i)\bigr]_{i \in \rmO}}{\mathrm{inf}} \ \mathbb{E}_{
    X_\rmO \sim P_d (X_\rmO),
    \mathrm{PA}_{X_\rmO} \sim \phi(X_\rmO)
    } 
    \bigl[ c \bigl(X_\rmO, \psi_{\theta}(\mathrm{PA}_{X_\rmO}) \bigr) \bigr],   
\end{equation}
where $\mathfrak{C}\left(X_{i}\right)=\left\{ \phi_{i}:\phi_{i}\#P_{d}\left(X_{i}\right)=P_{\theta}\left(\mathrm{PA}_{X_{i}}\right)\right\} ,\forall i\in\mathbf{O}$.

If we make some further assumptions including: (i) the family model distributions $P_\theta, \theta \in \Theta$ induced by the graphical model is sufficiently rich to contain the data distribution, meaning that there exist $\theta^* \in \Theta$ such that $P_{\theta^*}(X_{\mathbf{O}}) = P_d(X_{\mathbf{O}})$ and (ii) the family of backward maps $\phi_i, i \in \mathbf{O}$ has infinite capacity (i.e., they include all measure functions), the infimum really peaks $0$ at an optimal backward maps $\phi_i^*, i \in \mathbf{O}$. We thus can replace the infimum by a minimization as
\begin{equation}
\label{eq:min_main}
\underset{\bigl[\phi_i \in \mathfrak{C}(X_i)\bigr]_{i \in \rmO}}{\mathrm{min}} \ \mathbb{E}_{
    X_\rmO \sim P_d (X_\rmO),
    \mathrm{PA}_{X_\rmO} \sim \phi(X_\rmO)
    } 
    \bigl[ c \bigl(X_\rmO, \psi_{\theta}(\mathrm{PA}_{X_\rmO}) \bigr) \bigr].
\end{equation}

To make the OP in (\ref{eq:min_main}) tractable for training, we do relaxation as
\begin{equation}
\min_{\phi}\left\{ \mathbb{E}_{X_{\mathbf{O}}\sim P_{d}(X_{\mathbf{O}}),\mathrm{PA}_{X_{\mathbf{O}}}\sim\phi(X_{\mathbf{O}})}\left[c\bigl(X_{\mathbf{O}},\psi_{\theta}(\mathrm{PA}_{X_{\mathbf{O}}})\bigr)\right]+\eta D\left(P_{\phi},P_{\theta}\left(\mathrm{PA}_{X_{\mathbf{O}}}\right)\right)\right\} ,\label{eq:relax}
\end{equation}
where $\eta >0$, $P_\phi$ is the distribution induced by the backward maps, and $D$ represents a general divergence. Here we note that $D\left(P_{\phi},P_{\theta}\left(\mathrm{PA}_{X_{\mathbf{O}}}\right)\right)$ can be decomposed into
\[
D\left(P_{\phi},P_{\theta}\left(\mathrm{PA}_{X_{\mathbf{O}}}\right)\right)=\sum_{i\in\mathbf{O}}D_i\left(P_{\phi_{i}},P_{\theta}\left(\mathrm{PA}_{X_{\mathbf{i}}}\right)\right),
\]
which is the sum of the divergences between the specific backward map distributions and their corresponding model distributions on the parent nodes (i.e., $P_{\phi_{i}}=\phi_{i}\#P_{d}\left(X_{i}\right)$). Additionally, in practice, using the WS distance for $D_i$ leads to the following OP
\begin{equation}
\min_{\phi}\left\{ \mathbb{E}_{X_{\mathbf{O}}\sim P_{d}(X_{\mathbf{O}}),\mathrm{PA}_{X_{\mathbf{O}}}\sim\phi(X_{\mathbf{O}})}\left[c\bigl(X_{\mathbf{O}},\psi_{\theta}(\mathrm{PA}_{X_{\mathbf{O}}})\bigr)\right]+\eta\sum_{i\in\mathbf{O}} \textrm{W}_{c_{i}}\left(P_{\phi_{i}},P_{\theta}\left(\mathrm{PA}_{X_{i}}\right)\right)\right\} .\label{eq:relax-WS}
\end{equation}

The following theorem characterizes the ability to search the optimal solutions for the OPs in (\ref{eq:min_main}), (\ref{eq:relax}), and (\ref{eq:relax-WS}).

\begin{theorem} Assume that the family model distributions $P_\theta, \theta \in \Theta$ induced by the graphical model is sufficiently rich to contain the data distribution, meaning that there exist $\theta^* \in \Theta$ such that $P_{\theta^*}(X_{\mathbf{O}}) = P_d(X_{\mathbf{O}})$ and the family of backward maps $\phi_i, i \in \mathbf{O}$ has infinite capacity (i.e., they include all measure functions).  The OPs in (\ref{eq:min_main}), (\ref{eq:relax}), and (\ref{eq:relax-WS}) are equivalent and can obtain the common optimal solution.
\end{theorem}

\begin{proof}
Let $\theta^* \in \Theta$ be the optimal solution such that $P_{\theta^*}(X_{\mathbf{O}}) = P_d(X_{\mathbf{O}})$ and $W_{c}\left(P_{d}\left(X_{\mathbf{O}}\right),P_{\theta^*}\left(X_{\mathbf{O}}\right)\right) = 0$. Let $\Gamma^* \in \mathcal{P}(P_d(X_{\mathbf{O}}),P_\theta(X_{\mathbf{O}}))$ be the optimal joint distribution over $P_d(X_{\mathbf{O}})$ and $P_\theta(X_{\mathbf{O}})$ of the corresponding Wasserstein distance, meaning that if $(X_{\mathbf{O}}, \Tilde{X}_{\mathbf{O}}) \sim \Gamma^*$ then $X_{\mathbf{O}} = \Tilde{X}_{\mathbf{O}}$. Using the gluing lemma as in the previous theorem, there exists a joint distribution $\alpha^*$ over $A \times B \times C$ such that $\alpha^*_{AC} = (\pi_A, \pi_C)\#\alpha^* = \Gamma^*$ and $\alpha^*_{BC} = (\pi_B, \pi_C)\#\alpha^* = \beta^*$ where $\beta^*=(id, \psi_\theta)\# P_\theta^*([\textrm{PA}_{X_i}]_{i\in \mathbf{O}})$ is a deterministic coupling or joint distribution over $P_\theta([\textrm{PA}_{X_i}]_{i\in \mathbf{O}})$ and $P_\theta^*(X_{\mathbf{O}})$. This follows that $\alpha^*$ consists of the sample $(X_\mathbf{O}, \mathrm{PA}_{X_\mathbf{O}}, X_\mathbf{O})$ where $\psi_{\theta^*}(\mathrm{PA}_{X_\mathbf{O}})= X_\mathbf{O}$ with $X_\mathbf{O} \sim P_d(X_\mathbf{O})= P_\theta^*(X_\mathbf{O})$.

Let us denote $\gamma^* = (\pi_A, \pi_B)\#\alpha^*$ as a joint distribution over $P_d(X_{\mathbf{O}})$ and  $P_\theta^*([\textrm{PA}_{X_i}]_{i\in \mathbf{O}})$. Let $\gamma_i^*, i \in \mathbf{O}$ as the restriction of $\gamma^*$ over $P_d(X_{i})$ and $P_\theta^*(\textrm{PA}_{X_i})$. Let $\phi^*_i, i \in \mathbf{O}$ be the functions in the family of the backward functions that can well-approximate $\gamma^*_i, i \in \mathbf{O}$ (i.e., $\phi^*_i = \gamma^*_i, i \in \mathbf{O}$). For any $X_\mathbf{O} \sim P_d(X_\mathbf{O})$, we have for all $i \in \mathbf{O}$, $\mathrm{PA}_{X_i} = \phi^*_i(X_i)$ and $\psi_{\theta^*}(\mathrm{PA}_{X_i}) = X_i$. These imply that (i) $\mathbb{E}_{X_{\mathbf{O}}\sim P_{d}(X_{\mathbf{O}}),\mathrm{PA}_{X_{\mathbf{O}}}\sim\phi^*(X_{\mathbf{O}})}\left[c\bigl(X_{\mathbf{O}},\psi_{\theta^*}(\mathrm{PA}_{X_{\mathbf{O}}})\bigr)\right] =0$ and (ii) $P_{\phi^*_{i}} = P_{\theta^*}\left(\mathrm{PA}_{X_{i}}\right), \forall i \in \mathbf{O}$, which further indicate that the OPs in (\ref{eq:min_main}), (\ref{eq:relax}), and (\ref{eq:relax-WS}) are minimized at $0$ with the common optimal solution $\phi^*$ and $\theta^*$.         
\end{proof}

\section{OTP-DAG as a generalization of WAE}\label{sup:ae_relation}

\begin{figure}[hbt!]
\centering
     \begin{subfigure}
         \centering\includegraphics[width=0.40\linewidth]{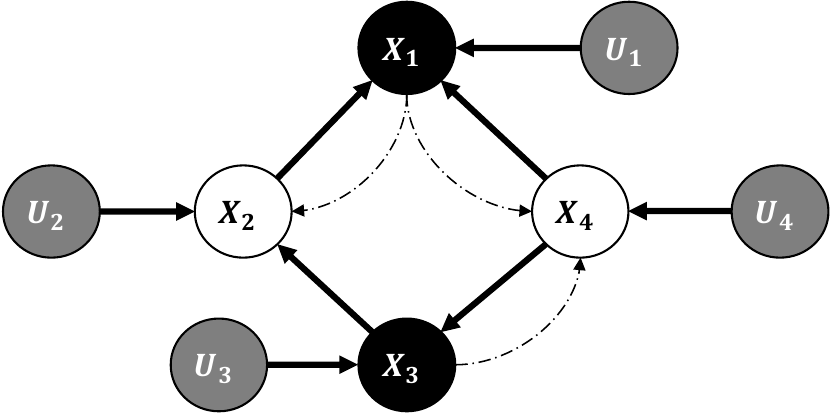}
     \end{subfigure}~
     \hspace{5mm}
     \begin{subfigure}
         \centering
         \includegraphics[width=0.10\linewidth]{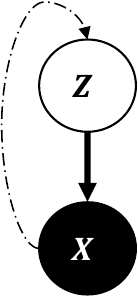}         
     \end{subfigure}
     \caption{\textbf{(Left)} Algorithmic DAG. \textbf{(Right)} Standard Auto-encoder.} 
     \label{fig:wae}
\end{figure}

Figure \ref{fig:wae} sheds light on an interesting connection of our method with auto-encoding models. Considering a graphical model of only two nodes: the observed node $X$ and its latent parent $Z$, we define a backward map $\phi$ over $X$ such that $\phi \# P_d(X) = P_{\theta}(Z)$ where $P_{\theta}(Z)$ is the prior over $Z$. The backward map can be viewed as a (stochastic) encoder approximating the prior $P_{\theta}(Z)$ with $P_{\phi}(Z) := \mathbb{E}_{X}[\phi(Z \vert X)]$. OTP-DAG now reduces to Wasserstein auto-encoder WAE \citep{tolstikhin2017wasserstein}, where the forward mapping $\psi$ plays the role of the decoder. OTP-DAG therefore serves as a generalization of WAE for learning a more complex structure where there is the interplay of more parameters and hidden variables. 

In this simplistic case, our training procedure is precisely as follows:

\begin{enumerate}[noitemsep]
    \item Draw $x \sim P_d(X)$.
    \item Draw $z \sim \phi(Z \vert X)$. 
    \item Draw $\widetilde{x} \sim P_{\theta}(X | Z)$.
    \item Update $\phi$ and $\theta$ alternately by descending objective (\ref{eq:finalobj}).
\end{enumerate}

Our cost function explicitly minimizes two terms:
(1) the push-forward divergence $D[P_{\phi}(Z), P_{\theta}(Z)]$ where $D$ is an arbitrary divergence, and (2) the reconstruction loss between $x$ and $\widetilde{x}$. At a high level, our learning dynamic ensures $\phi_\# P_d(X) = P_\phi(Z) = P_{\theta}(Z)$ so that $z \sim P_\phi(Z)$ follows the prior distribution $P_{\theta}(Z)$. However, such samples cannot ignore information in the input $X$ because we need $\widetilde{x} \sim P_\theta(X \mid Z = z)$ to effectively reconstruct the observed samples.

There might also be a concern that relaxing the push-forward constraint into the divergence term means the backward $\phi$ is forced to mimic the prior, which may lead to a situation similar to \textbf{posterior collapse} notoriously occurring to VAE. We here detail why VAE is prone to this issue and how the OT-based objective mitigates it. 

\paragraph{1. The push-forward divergence:} While the objectives of OTP-DAG/WAE and VAE entail the prior matching term. the two formulations are different in nature. 

Let $Q$ denote the set of variational distributions. If we consider $\phi$ as an encoder, the VAE objective can be written as 
\begin{equation}\label{eq:vae}
    \inf_{\phi(Z|X) \in Q} \mathbb{E}_{X \sim P(X)} [ D_{\textrm{KL}}(\phi(Z|X), P_{\theta}(Z))] - \mathbb{E}_{Z \sim \phi(Z|X)} [\log P_{\theta}(X | Z)].
\end{equation}

By minimizing the above KL divergence term,  VAE basically tries to match the prior $P(Z)$ for all different examples drawn from $P_d(X)$. Under the VAE objective, it is thus easier for $\phi$ to collapse into a distribution independent of $P_d(X)$, where specifically latent codes are close to each other and reconstructed samples are concentrated around only few values.  

For OTP-DAG/WAE, the regularizer in fact penalizes the discrepancy between $P_{\theta}(Z)$ and $P_{\phi}(Z) := \mathbb{E}_{X}[\phi(Z \vert X)] $, which can be optimized using GAN-based, MMD-based or Wasserstein distance. The latent codes of different examples $x \sim P_d(X)$ can lie far away from each other, which allows the model to maintain the dependency between the latent codes and the input. Therefore, it is more difficult for $\phi$ to mimic the prior and trivially satisfy the push-forward constraint. We refer readers to \citet{tolstikhin2017wasserstein} for extensive empirical evidence. 

\paragraph{2. The reconstruction loss:}
At some point of training, there is still a possibility to land at $\phi$ that yields samples $Z$ independent of input $X$. If this occurs, $\phi \# \delta x^{(1)}_c = \phi \# \delta x^{(2)}_c = P(Z)$ for any points $x^{(1)}_c, x^{(2)}_c \sim P_d(X)$. This means $\text{supp}(\phi(X^{1})) = \text{supp}(\phi(X^{2})) = \text{supp}(P_{\theta}(Z))$, so it would result in a very large reconstruction loss because it requires to map $\text{supp}(P_{\theta}(Z))$ to various $X^{1}$ and $X^{2}$. Thus our reconstruction term would heavily penalizes this. In other words, this term explicitly encourages the model to search for $\theta$ that helps reconstruction, thus preventing the model from converging to the backward $\phi$ that produces sub-optimal ancestral samples.  

Meanwhile, for VAE, if the family $Q$ contains all possible conditional distribution $\phi(Z | X)$, its objective is essentially to maximize the marginal log-likelihood $\mathbb{E}_{P(X)}[\log P_{\theta}(X)]$, or minimize the KL divergence $\text{KL}(P_d, P_{\theta})$. It is shown in \citet{dai2020usual} that under posterior collapse, VAE produces poor reconstructions yet the loss can still decrease i.e achieve low negative log-likelihood scores and still able to assign high-probability to the training data. 

In summary, it is such construction and optimization of the backward that prevents OTP-DAG from posterior collapse situation. We here search for $\phi$ within a family of measurable functions and in practice approximate it with deep neural networks. It comes down to empirical decisions to select the architecture sufficiently expressive in each application.

\section{Experimental setup}\label{sup:exp}
In the following, we explain how OTP-DAG algorithm is implemented in practical applications, including how to reparameterize the model distribution, to design the backward mapping and to define the optimization objective. We also here provide the training configurations for our method and the baselines. All models are run on $4$ RTX 6000 GPU cores using Adam optimizer with a fixed learning rate of $1e-3$. Our code is published at \url{https://github.com/isVy08/OTP}.

\subsection{Latent Dirichlet Allocation}\label{sup:lda_exp}
For completeness, let us recap the model generative process. We consider a corpus $\gD$ of $M$ independent documents where each document is a sequence of $N$ words denoted by $W_{1:N} = (W_1, W_2, \cdots, W_N)$. Documents are represented as random mixtures over $K$ latent topics, each of which is characterized by a distribution over words. Let $V$ be the size of a vocabulary indexed by $\{1, \cdots, V\}$. Latent Dirichlet Allocation (LDA) \cite{blei2003latent} dictates the following generative process for every document in the corpus:
\begin{enumerate}[noitemsep]
    \item Choose $\theta \sim \textrm{Dir}(\alpha)$,
    \item Choose $\gamma_k \sim \textrm{Dir}(\beta)$ where $k \in \{1, \cdots, K\}$,
    \item For each of the word positions $n \in \{1, \cdots, N\}$,
    \begin{itemize}
        \item Choose a topic $z_n \sim \textrm{Multi-Nominal}(\theta)$,
        \item Choose a word $w_n \sim \textrm{Multi-Nominal}(\gamma_kn)$,
    \end{itemize}
\end{enumerate}
where $\textrm{Dir}(.)$ is a Dirichlet distribution, $\alpha < 1$ and $\beta$ is typically sparse. $\theta$ is a $K-$dimensional vector that lies in the $(K-1)-$simplex and $\gamma_k$ is a $V-$dimensional vector represents the word distribution corresponding to topic $k$. Throughout the experiments, $K$ is fixed at $10$. 

\paragraph{Parameter estimation.}
We consider the topic-word distribution $\gamma$ as a fixed quantity to be estimated. $\gamma$ is a $K \times V$ matrix where $\gamma_{kn} := P(W_n = 1 \vert Z_n =1)$. The learnable parameters therefore consist of $\gamma$ and $\alpha$. An input document is represented with a $N \times V$ matrix where a word $W_i$ is represented with a one-hot $V-$vector such that the value at the index $i$ in the vocabulary is $1$ and $0$ otherwise. Given $\gamma \in [0,1]^{K \times V}$ and a selected topic $k$, the deterministic forward mapping to generate a document $W$ is defined as  

\begin{equation*}
    W_{1:N} = \psi_{\gamma}(Z) = \text{Cat-Concrete}\big(\text{softmax}(Z'\gamma)\big),
\end{equation*}

where $Z \in \{0, 1\}^K$ is in the one-hot representation (i.e., $Z^k = 1$ if state $k$ is the selected and $0$ otherwise) and $Z'$ is its transpose. By applying the Gumbel-Softmax trick \cite{jang2016categorical,maddison2016concrete}, we re-parameterize the Categorical distribution into a function $\text{Cat-Concrete}(.)$ that takes the categorical probability vector (i.e., sum of all elements equals $1$) and output a relaxed probability vector. To be more specific, given a categorical variable of $K$ categories with probabilities $\big[ p_1, p_2, ..., p_K \big]$, for every the $\text{Cat-Concrete}(.)$ function is defined on each $p_k$ as 
\begin{equation*}
    \text{Cat-Concrete}(p_k) = \frac{\exp \big\{ (\log p_k + G_k)/\tau \big\} }{\sum^{K}_{k=1} \exp 
\big\{ (\log p_k + G_k)/\tau \big\}  },
\end{equation*}

with temperature $\tau$, random noises $G_k$ independently drawn from Gumbel distribution $G_t =-\log(- \log u_t), \ u_t \sim \textrm{Uniform}(0,1)$. 

We next define a backward map that outputs for a document a distribution over $K$ topics by $\phi(Z \mid W_{1:N}) = \textrm{Cat}(Z).$ Given observations $W_{1:N}$, our learning procedure begins by sampling $\widetilde{Z} \sim \phi(Z \vert W_{1:N})$ and pass $\widetilde{Z}$ through the generative process given by $\psi$ to obtain the reconstruction. Notice here that we have a prior constraint over the distribution of $\theta$ i.e., $\theta$ follows a Dirichlet distribution parameterized by $\alpha$. This translates to a push forward constraint in order to optimize for $\alpha$. To facilitate differentiable training, we use softmax Laplace approximation \citep{mackay1998choice,srivastava2017autoencoding} to approximate a Dirichlet distribution with a softmax Gaussian distribution. The relation between $\alpha$ and the Gaussian parameters $\big( \mu_k, \Sigma_k \big)$ w.r.t a category $k$ where $\Sigma_k$ is a diagonal matrix is given as 

\begin{equation}\label{eq:softGau}
    \mu_k (\alpha) = \log \alpha_k - \frac{1}{K} \sum_{i=1}^K \log \alpha_i, \quad \Sigma_k (\alpha) = \frac{1}{\alpha_k} \bigg( 1 - 
    \frac{2}{K} \bigg) + \frac{1}{K^2} \sum_{i=1}^K \frac{1}{\alpha_i}.
\end{equation}

Let us denote $P_{\alpha} := \gN\big( \mu(\alpha), \Sigma(\alpha) \big) \approx \textrm{Dir}(\alpha)$ with $\mu = [\mu_k]_{k=1}^{K}$ and $\Sigma = [\Sigma_k]_{k=1}^{K}$ defined as above. The optimization objective is given as 

\begin{equation*}\label{eq:lda}
     \min_{\alpha, \gamma} \quad \mathbb{E}_{W_{1:N} \sim \gD, \widetilde{Z} \sim \phi(Z \mid W_{1:N})} c \big[ W_{1:N}, \psi_{\gamma}(\widetilde{Z}) \big]  + \eta \ D_{\textrm{WS}} \big( P_{\phi}(Z), P_{\alpha}(Z) \big),
\end{equation*}

where $c$ is cross-entropy loss function and $D_{\textrm{WS}}$ is exact Wasserstein distance\footnote{\url{https://pythonot.github.io/index.html}}. The sampling process $\theta \sim P_{\alpha}$ is also relaxed using standard Gaussian reparameterization trick whereby $\theta = \mu(\alpha) + u \Sigma(\alpha)$ with $u \sim \gN(0,1)$.

\paragraph{Remark.} Our framework in fact learns both $\alpha$ and $\gamma$ at the same time. Our estimates for $\alpha$ (averaged over $K$) are nearly $100\%$ faithful at $0.10, 0.049, 0.033$ (recall that the ground-truth $\alpha$ is uniform over $K$ where $K = 10, 20, 30$ respectively). Figure \ref{fig:lda_convergence} shows the convergence behavior of OTP-DAG during training where our model converges to the ground-truth patterns relatively quickly. 

Figures \ref{fig:topic20} and \ref{fig:topic30} additionally present the topic distributions of each method for the second and third synthetic sets. We use horizontal and vertical patterns in different colors to distinguish topics from one another. Red circles indicate erroneous patterns. Note that these configurations are increasingly more complex, so it requires more training time for all methods to achieve better performance. Although our method may exhibit some inconsistencies in recovering accurate word distributions for each topic, these discrepancies are comparatively less pronounced when compared to the baseline methods. This observation indicates a certain level of robustness in our approach.

\begin{figure}[h!]
\centering
\includegraphics[width=0.70\linewidth]{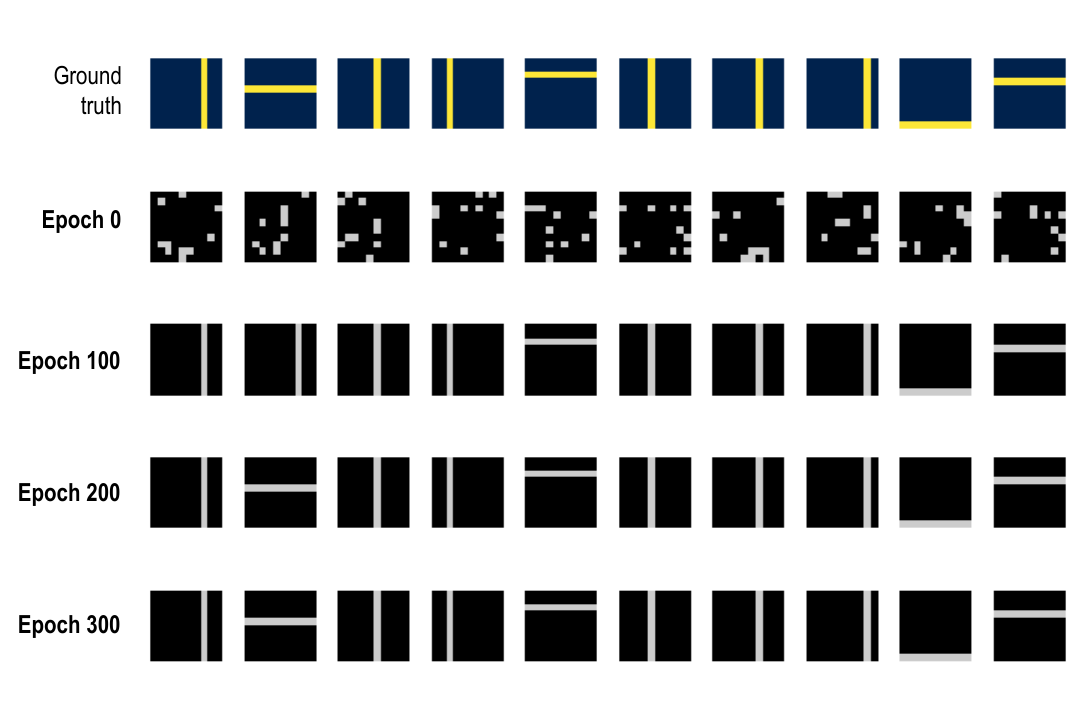}
\vspace{-1em}
\caption{Converging patterns of $10$ random topics from our OTP-DAG after $100, 200, 300$ iterations.}
\label{fig:lda_convergence}
\end{figure}

\begin{figure}[h!]
    \centering
    \includegraphics[width=\linewidth]{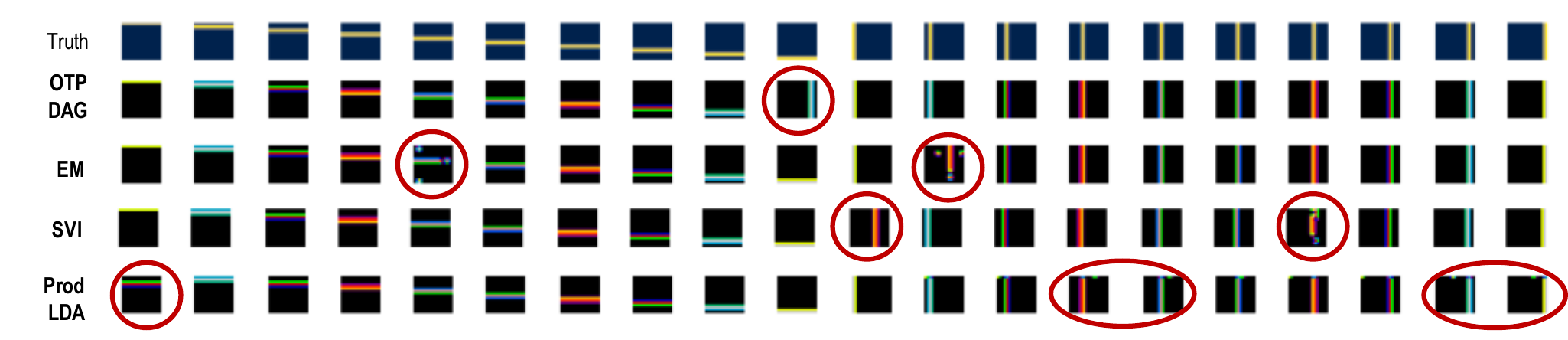}
    \caption{Topic-word distributions inferred by each method from the second set of synthetic data after 300 training epochs.}
    \label{fig:topic20}
\end{figure}

\begin{figure}[h!]
    \centering
    \includegraphics[width=\linewidth]{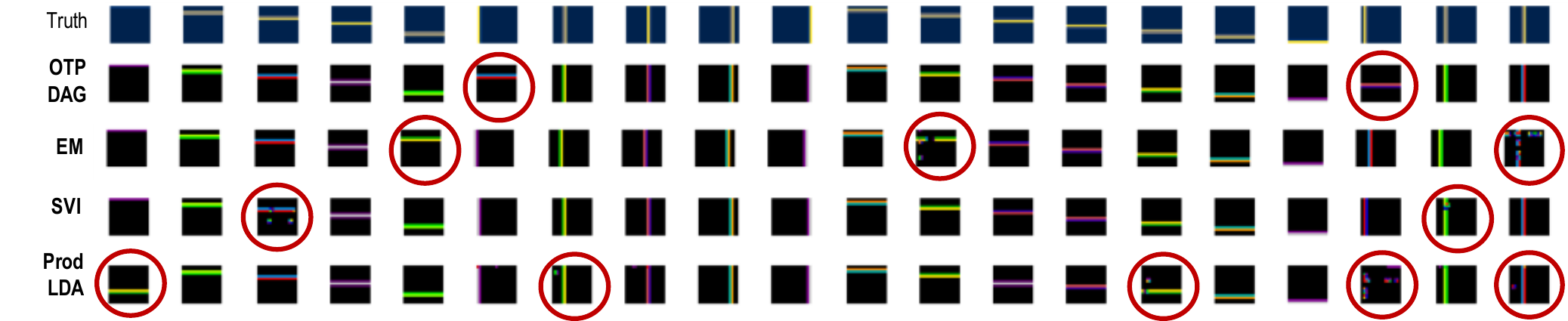}
    \caption{Topic-word distributions inferred by OTP-DAG from the third set of synthetic data after 300 training epochs.}
    \label{fig:topic30}
\end{figure}

\paragraph{Topic inference.}
In this experiment, we apply OTP-DAG on real-world datasets. We here revert to the original generative process where the topic-word distribution follows a Dirichlet distribution parameterized by the concentration parameters $\beta$, instead of having $\gamma$ as a fixed quantity. In this case, $\beta$ is initialized as a matrix of real values i.e., $\beta \in \mathbb{R}^{K \times V}$ representing the log concentration values. The forward process is given as
\begin{equation*}
    W_{1:N} = \psi_{\gamma}(Z) = \text{Cat-Concrete}\big(\text{softmax}(Z'\gamma)\big),
\end{equation*}

where $\gamma_k = \mu_k \big( \exp(\beta_k) \big) + u_k \Sigma_k \big(  \exp(\beta_k)\big)$ and $u_k \sim \gN(0,1)$ is a Gaussian noise. This is realized by using softmax Gaussian trick as in Eq. (\ref{eq:softGau}), then applying standard Gaussian reparameterization trick. The optimization procedure follows the previous application.

\begin{table}[!h]
\caption{Qualitive evaluation of the topics inferred for $3$ real-world datasets.}
\vspace{0.5em}
    \centering
    \resizebox{\linewidth}{!}{
    \begin{tabular}{l | l}
    \toprule
    \multicolumn{2}{c}{20 News Group} \\
    \midrule
    Topic 1 & \textit{car, bike, front, engine, mile, ride, drive, owner, road, buy}\\
    Topic 2&\textit{game, play, team, player, season, fan, win, hit, year, score}\\
    Topic 3&\textit{government, public, key, clipper, security, encryption, law, agency, private, technology}\\
    Topic 4&\textit{religion, christian, belief, church, argument, faith, truth, evidence, human, life}\\
    Topic 5&\textit{window, file, program, software, application, graphic, display, user, screen, format}\\
    Topic 6&\textit{mail, sell, price, email, interested, sale, offer, reply, info, send}\\
    Topic 7&\textit{card, drive, disk, monitor, chip, video, speed, memory, system, board}\\
    Topic 8&\textit{kill, gun, government, war, child, law, country, crime, weapon, death}\\
    Topic 9&\textit{make, time, good, people, find, thing, give, work, problem, call}\\
    Topic 10&\textit{fire, day, hour, night, burn, doctor, woman, water, food, body}\\
        \midrule
    \multicolumn{2}{c}{BBC News} \\ 
    \midrule
    Topic 1 & \textit{rise, growth, market, fall, month, high, economy, expect, economic, price}\\
    Topic 2 & \textit{win, play, game, player, good, back, match, team, final, side}\\
    Topic 3 & \textit{user, firm, website, computer, net, information, software, internet, system, technology} \\
    Topic 4 & \textit{technology, market, digital, high, video, player, company, launch, mobile, phone}\\
   Topic 5 & \textit{election, government, party, labour, leader, plan, story, general, public, minister}\\
    Topic 6 & \textit{film, include, star, award, good, win, show, top, play, actor}\\
    Topic 7 & \textit{charge, case, face, claim, court, ban, lawyer, guilty, drug, trial}\\
    Topic 8 & \textit{thing, work, part, life, find, idea, give, world, real, good}\\
    Topic 9 & \textit{company, firm, deal, share, buy, business, market, executive, pay, group}\\
    Topic 10 & \textit{government, law, issue, spokesman, call, minister, public, give, rule, plan}\\
    \midrule
    \multicolumn{2}{c}{DBLP} \\
    \midrule
    Topic 1&\textit{learning, algorithm, time, rule, temporal, logic, framework, real, performance, function}\\
    Topic 2&\textit{efficient, classification, semantic, multiple, constraint, optimization, probabilistic, domain, process, inference}\\
    Topic 3&\textit{search, structure, pattern, large, language, web, problem, representation, support, machine}\\
    Topic 4&\textit{object, detection, application, information, method, estimation, multi, dynamic, tree, motion}\\
    Topic 5&\textit{system, database, query, knowledge, processing, management, orient, relational, expert, transaction}\\
    Topic 6&\textit{model, markov, mixture, variable, gaussian, topic, hide, latent, graphical, appearance}\\
    Topic 7&\textit{network, approach, recognition, neural, face, bayesian, belief, speech, sensor, artificial}\\
    Topic 8&\textit{base, video, content, code, coding, scalable, rate, streaming, frame, distortion}\\
    Topic 9&\textit{datum, analysis, feature, mining, cluster, selection, high, stream, dimensional, component}\\
    Topic 10&\textit{image, learn, segmentation, retrieval, color, wavelet, region, texture, transform, compression}\\
    
    \bottomrule
    \end{tabular}  
    }    
    \label{tab:lda_quali}
\end{table}

\paragraph{Training configuration.} The underlying architecture of the backward maps consists of an LSTM and one or more linear layers. We train all models for $300$ and $1,000$ epochs with batch size of $50$ respectively for the $2$ applications. We also set $\tau = 1.0, 2.0$ and $\eta = 1e-4, 1e-1$ respectively. 

\subsection{Hidden Markov Models}\label{sup:hmm_exp}
We here attempt to learn a Poisson hidden Markov model underlying a data stream. Given a time series $\gD$ of $T$ steps, the task is to segment the data stream into $K$ different states, each of which is associated with a Poisson observation model with rate $\lambda_k$. The observation at each step $t$ is given as 

\begin{equation*}
    P(X_t \vert Z_t = k) = \mathrm{Poi}(X_t \vert \lambda_k), \quad \text{for} \ k=1, \cdots, K.
\end{equation*}

The Markov chain stays in the current state with probability $p$ and otherwise transitions to one of the other $K-1$ states uniformly at random. The transition distribution is given as 
\begin{equation*}
    z_1 \sim \mathrm{Cat} \bigg( \bigg\{ \frac{1}{4}, \frac{1}{4}, \frac{1}{4}, \frac{1}{4} \bigg\} \bigg), \quad  z_t \vert z_{t-1} \sim \mathrm{Cat} \bigg( 
    \left\{\begin{array}{lr}
        \pi & \text{if } Z_t = Z_{t-1}\\
        \frac{1-\pi}{4-1} & \text{otherwise } 
        \end{array}
        \right\}\bigg)
\end{equation*}

We first apply Gaussian reparameterization on each Poisson distribution, giving rise to a deterministic forward mapping 

$$X_t = \psi_t (Z_t) = Z'_t \exp(\lambda) + u_t \sqrt{Z_t \exp(\lambda)},$$

where $\lambda \in \mathbb{R}^K$ is the learnable parameter vector representing log rates, $u_k \sim \gN(0,1)$ is a Gaussian noise, $Z_t \in \{0,1\}^K$ is in the one-hot representation and $Z'_t$ is its transpose. We define a global backward map $\phi$ that outputs the distributions for individual $Z_t$ as $\phi(Z_t \vert X_t) := \mathrm{Cat}(Z_t \vert X_t).$

The first term in the optimization object is the reconstruction error given by a cost function $c$. The push forward constraint ensures the backward probabilities for the state variables align with the prior transition distributions. Denoting $\psi = [\psi_t]^T_{t=1}$, we learn $\lambda_{1:K}$ by optimizing the following objective

\begin{align*}\label{eq:hmm}
    \min_{\lambda} \quad \mathbb{E}_{X_{1:T} \sim \gD, \widetilde{Z}_{1:T}\sim \phi(Z_{1:T} \vert X_{1:T})}  c \big[ X_{1:T}, \psi(\widetilde{Z}_{1:T}) \big] + \eta \ D_{\textrm{WS}} \big[ P_{\phi} (Z_{1:T}), P_{\pi}(Z_{1:T}) \big].
\end{align*}

In the experiment, we choose $T = 200$ and smooth $L_1$ loss \citep{girshick2015fast} is chosen as the cost function. $D_{\textrm{WS}}$ is exact Wasserstein distance with KL divergence as the ground cost. 

\paragraph{Training configuration.} The underlying architecture of the backward map is a $3-$ layer fully connected perceptron. The Poisson HMM is trained for $50$ epochs with $\eta = 0.1$ and $\tau = 0.1$.

\begin{figure}[hbt!]
\centering
     \begin{subfigure}
         \centering
         \includegraphics[width=0.45\linewidth]{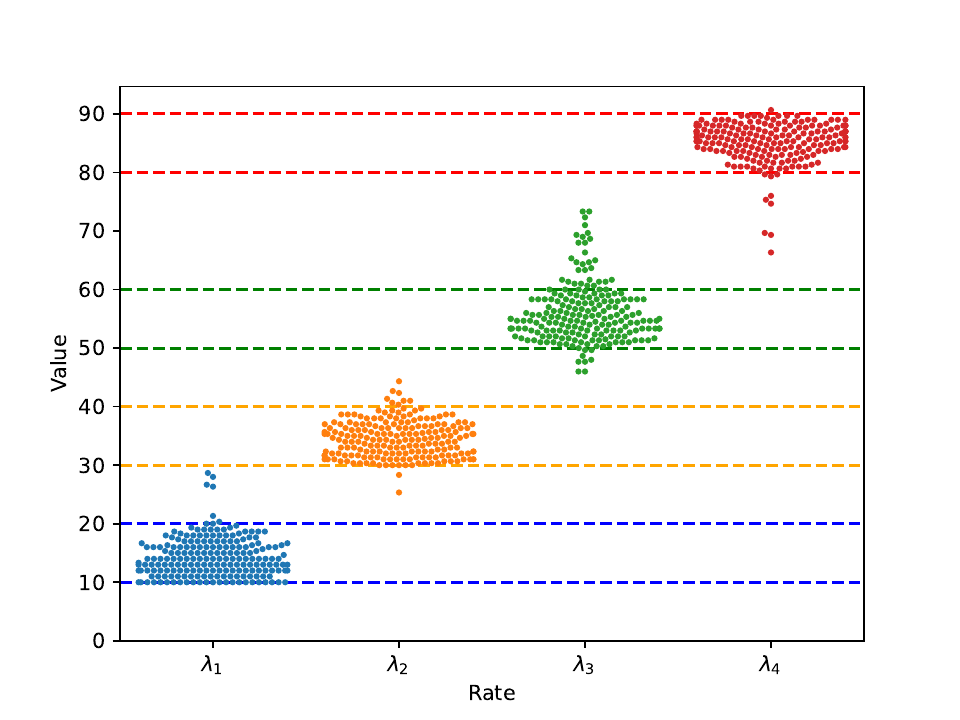}
     \end{subfigure}~
     \hspace{5mm}
     \begin{subfigure}
         \centering
         \includegraphics[width=0.45\linewidth]{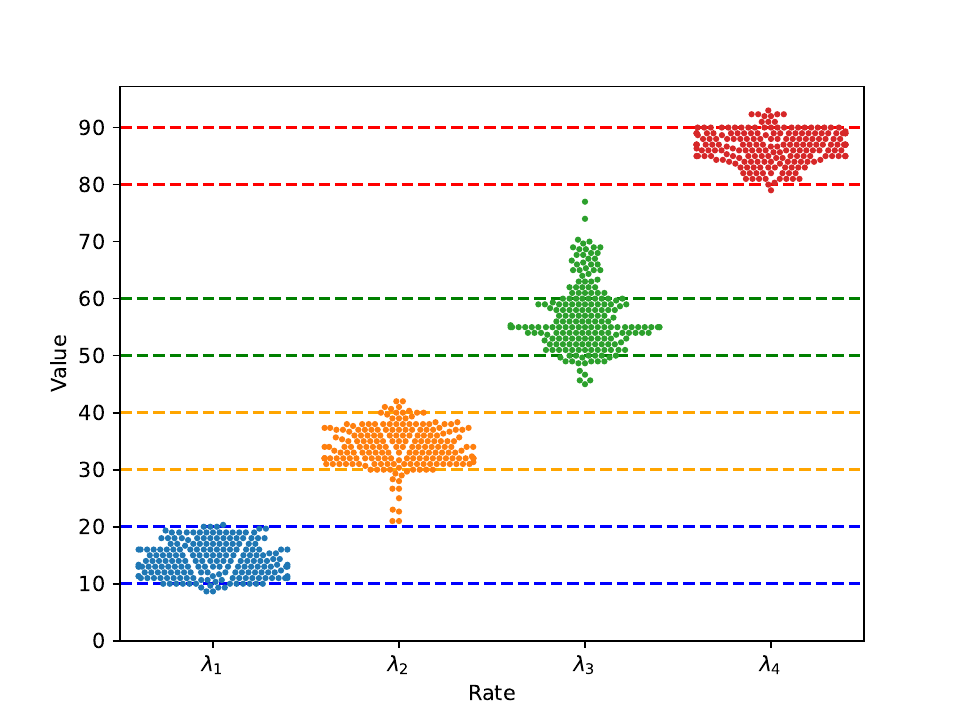}
     \end{subfigure}
     \caption{Distribution of the estimates of the concentration parameters $\lambda_{1:4}$ from OTP-DAG \textbf{(Left)} and EM \textbf{(Right)}. Our estimated distribution aligns more uniformly with the true generative process.} 
     \label{fig:estimate_dist}
\end{figure}




\subsection{Learning Discrete Representations}\label{sup:repl_exp}

To understand vector quantized models, let us briefly review Quantization Variational Auto-Encoder (VQ-VAE) \cite{van2017neural}. The practical setting of VQ-VAE in fact considers a $M-$dimensional discrete latent space $\gC^{M} \in \mathbb{R}^{M \times D}$ that is the $M-$ary Cartesian power of $\gC$ with $\gC = \{c_k\}_{k=1}^K \in \mathbb{R}^{K \times D}$ i.e., $\gC$ here is the set of learnable latent embedding vectors $c_k$. The latent variable $Z = [Z^m]^{M}_{m=1}$ is an $M-$component vector where each component $Z^m \in \gC$. VQ-VAE is an encoder-decoder, in which the encoder $f_e: \gX \mapsto \mathbb{R}^{M \times D}$ maps the input data $X$ to the latent representation $Z$ and the decoder $f_d: \mathbb{R}^{M \times D} \mapsto \gX$ reconstructs the input from the latent representation. However, different from standard VAE, the latent representation used for reconstruction is discrete, which is the projection of $Z$ onto $\gC^{M}$ via the quantization process $Q$. Let $\bar{Z}$ denote the discrete representation. The quantization process is modeled as a deterministic categorical posterior distribution such that
\begin{equation*}
    \bar{Z}^{m} = Q(Z^m) = c_k, 
\end{equation*}
where $k = \underset{k}{\textrm{argmin}} \  d\big(Z^m, c_k\big)$, $Z^m = f^{m}_{e}(X)$ and $d$ is a metric on the latent space. 

In our language, each vector $c_k$ can be viewed as the centroid representing each latent sub-space (or cluster). The quantization operation essentially searches for the closet cluster for every component latent representation $z^m$. VQ-VAE minimizes the following objective function:
\begin{equation*}
 \mathbb{E}_{x \sim \gD} \bigg[ 
    d_x \big[ f_d \big( Q(f_e(x)) \big), x \big]  + 
    d_z \big[ \textbf{sg} \big( f_e(x) \big), \bar{z}  \big] + 
    \beta d_z \big[ f_e(x), \textbf{sg} \big(\bar{z}\big) \big]
 \bigg],
 \end{equation*}
where $\gD$ is the empirical data, $\textbf{sg}$ is the stop gradient operation for continuous training, $d_x, d_z$ are respectively the distances on the data and latent space and $\beta$ is set between $0.1$ and $2.0$ in the original proposal \citep{van2017neural}. 

In our work, we explore a different model to learning discrete representations. Following VQ-VAE, we also consider $Z$ as a $M-$component latent embedding. 
On a $k^{th}$ sub-space (for $k \in \{1, \cdots, K\})$, we impose a Gaussian distribution parameterized by $\mu_k, \Sigma_k$ where $\Sigma_k$ is diagonal. We also endow $M$ discrete distributions over $\mathbf{C}^1,\dots,\mathbf{C}^M$, sharing a common support set as the set of sub-spaces induced by $\{(\mu_k,\Sigma_k)\}_{k=1}^K$:

$$\mathbb{P}_{k,\pi^m}= \sum_{k=1}^{K}\pi_{k}^{m}\delta_{\mu_{k}}, \ \text{for } m=1,\dots,M.$$

with the Dirac delta function $\delta$ and the weights $\pi^m\in\Delta_{K-1}= \{\alpha\geq \boldsymbol{0}: \Vert \alpha\Vert_1 =1\}$ in the $(K-1)$-simplex. The probability a data point $z^m$ belongs to a discrete $k^{th}$ sub-space follows a $K-$way categorical distribution $\pi^m = [\pi^m_1, \cdots, \pi^m_K]$. In such a practical setting, the generative process is detailed as follows
\begin{enumerate}[noitemsep]
    \item For $m \in \{1, \cdots, M\}$, 
    \begin{itemize}
        \item Sample $k \sim \textrm{Cat}(\pi^m)$,
        \item Sample $z^m \sim \gN(\mu_k, \Sigma_k)$,
        \item Quantize $\mu^m_k = Q(z^m)$, 
    \end{itemize}
    \item $x = \psi_{\theta}([z^m]_{m=1}^M, [\mu^m_k]_{m=1}^M)$.
\end{enumerate}

where $\psi$ is a highly non-convex function with unknown parameters $\theta$. $Q$ refers to the quantization of $[z^m]_{m=1}^M$ to $[\mu^m_k]_{m=1}^M$ defined as $\mu^m_k = Q(z^m)$ where $k = \mathrm{argmin}_{k} \ d_z \big(z^m; \mu_k \big)$ and $d_z = \sqrt{(z^m - \mu_k)^{T} \Sigma_{k}^{-1} (z^m - \mu_k)}$ is the Mahalanobis distance. 

The backward map is defined via an encoder function $f_e$ and quantization process $Q$ as 
\begin{equation*}
    \phi(x) = \big[ f_e(x), Q(f_e(x)) \big], \quad z =[z^m]_{m=1}^M = f_e(x), \quad [\mu^m_k]_{m=1}^M = Q(z).
\end{equation*}
 
The learnable parameters are $\{\pi, \mu, \Sigma, \theta\}$ with $\pi = [[\pi^{m}_{k}]_{m=1}^M]_{k=1}^{K}, \mu = [\mu_k]_{k=1}^K, \Sigma = [\Sigma_k]_{k=1}^K$. 

Applying OTP-DAG to the above generative model yields the following optimization objective:
\begin{align*}
\underset{\pi, \mu, \Sigma, \theta}{\mathrm{min}} &\quad
\mathbb{E}_{X\sim \gD} \bigg[ c \big[ X, \psi_{\theta}(Z, \mu_k) \big]   
\bigg] 
+ \frac{\eta}{M} \ \sum_{m=1}^{M} \big[
D_{\textrm{WS}} \big(P_{\phi}(Z^m), P(\widetilde{Z}^m)\big) + 
D_{\textrm{WS}} \big(P_{\phi}(Z^m),\mathbb{P}_{k,\pi^{m}}\big)
\big] \\
&+ \eta_r \ \sum^{M}_{m=1} D_{\textrm{KL}} \big( \pi^m, \gU_K \big),
\end{align*}

where $P_{\phi}(Z^m):=f_{e}^{m}\#P(X)$ given by the backward $\phi$, $P(\widetilde{Z}^m) = \sum^{K}_{k=1} \pi^m_k \gN(\widetilde{Z}^m \vert \mu_k, \Sigma_k)$ is the mixture of Gaussian distributions. The copy gradient trick \citep{van2017neural} is applied throughout to facilitate backpropagation. 

The first term is the conventional reconstruction loss where $c$ is chosen to be mean squared error. Minimizing the second term $D_{\textrm{WS}} \big(P_{\phi}(Z^m),P(\widetilde{Z}^m)\big)$ forces the latent representations to follow the Gaussian distribution $\gN(\mu^m_k, \Sigma^m_k)$. Minimizing the third term $D_{\textrm{WS}}\big(P_{\phi}(Z^m),\mathbb{P}_{k,\pi^{m}}\big)$ encourages every $\mu_k$ to become the clustering centroid of the set of latent representations $Z^m$ associated with it. Additionally, the number of latent representations associated with the clustering centroids are proportional to $\pi^m_{k},k=1,...,K$. Therefore, we use the fourth term $\sum^{M}_{m=1} D_{\textrm{KL}}\big( \pi^m, \gU_K \big)$ to guarantee every centroid is utilized.

\paragraph{Training configuration.} We use the same experiment setting on all datasets. The models have an encoder with two convolutional layers of stride $2$ and
filter size of $4 \times 4$ with ReLU activation, followed by $2$ residual blocks, which contained a $3 \times 3,$ stride $1$ convolutional layer with ReLU activation followed by a $1 \times 1$ convolution. The decoder was similar, with two of these residual blocks followed by two de-convolutional layers. The hyperparameters are: $D = M = 64, K = 512, \eta=1e-3, \eta_r=1.0$, batch size of $32$ and $100$ training epochs.

\paragraph{Evaluation metrics.} The evaluation metrics used include (1) \textbf{SSIM:} the patch-level structure similarity index, which evaluates the similarity between patches of the two images; (2) \textbf{PSNR:} the pixel-level peak signal-to-noise ratio, which measures the similarity between the original and generated image at the pixel level; (3) feature-level \textbf{LPIPS} \citep{zhang2018unreasonable}, which calculates the distance between the feature representations of the two images; (4) the dataset-level Fr'echlet Inception Distance \textbf{(FID)} \citep{heusel2017gans}, which measures the difference between the distributions of real and generated images in a high-dimensional feature space; and (5) \textbf{Perplexity:} the degree to which the latent representations $Z$ spread uniformly over $K$ sub-spaces i.e., all $K$ regions are occupied. Perplexity score is defined as $\exp\big({-\sum_{k=1}^K p_{c_k} \log p_{c_k}}\big)$  where $p_{c_k}= N_{c_k} \slash \sum_{i=1}^K N_{c_i}$ is the probability of the $i^{th}$ codeword being used and $N_{c_i}$ is the number of latent representations associated with the codeword $c_i$. Perplexity is maximized when the distribution over the codebooks is uniform, indicating that all codebooks are utilized equally by the model there is no posterior collapse. Thus, higher perplexity is preferred.

\paragraph{Additional experiment.}
We additionally investigate a recent model called SQ-VAE \citep{takida2022sq} proposed to tackle the issue of codebook utilization. Table \ref{tab:sqvae} reports the performance of SQ-VAE in comparison with our OTP-DAG. We significantly outperform SQ-VAE on Perplexity, showing that our model mitigates codebook collapse issue more effectively, while compete on par with this SOTA model across the other metrics. It is worth noting that our goal here is not to propose any SOTA model to discrete representation learning, but rather to demonstrate the applicability of OTP-DAG on various tasks, particular problems where traditional methods such as EM or mean-field VI cannot simply tackle.

\paragraph{Qualitative examples.}
We first present the generated samples from the CelebA dataset using Image transformer \citep{parmar2018image} as the generative model. From Figure~\ref{fig:quali_celeba}, it can be seen that the discrete representation from the our method can be effectively utilized for image generation with acceptable quality.

We additionally show the reconstructed samples from CIFAR10 dataset for qualitative evaluation. Figures \ref{fig:cifar-original}-\ref{fig:cifar-otp} illustrate that the reconstructions from OTP-DAG have higher visual quality than VQ-VAE. The high-level semantic features of the input image and colors are better preserved with OTP-DAG than VQ-VAE from which some reconstructed images are much more blurry. 

\begin{table*}[h!]
    \centering
 \caption{Quality of image reconstructions}
 \vskip 0.15in
\resizebox{0.90\linewidth}{!}{
    \begin{tabular}{lrcrrrrr}
        \toprule 
Dataset & Method & Latent Size & SSIM $\uparrow$ & PSNR $\uparrow$ & LPIPS $\downarrow$ & rFID $\downarrow$ & Perplexity $\uparrow$
\tabularnewline 
\midrule 
CIFAR10 & \textbf{SQ-VAE}  & 8 $\times$ 8 & 0.80 & \textbf{26.11} & 0.23 & \textbf{55.4} & 434.8  \tabularnewline
 & \textbf{OTP-DAG} (Ours)  & 8 $\times$ 8 & 0.80 & 25.40 & 0.23 & 56.5 & \textbf{498.6} \tabularnewline
 \midrule 
         
MNIST & \textbf{SQ-VAE}  & 8 $\times$ 8 & \textbf{0.99} & \textbf{36.25} & 0.01  & \textbf{3.2} & 301.8 \tabularnewline
& \textbf{OTP-DAG} (Ours)   & 8 $\times$ 8 & 0.98 & 33.62 & 0.01 & 3.3 & \textbf{474.6}
\tabularnewline
 \midrule 
SVHN & \textbf{SQ-VAE} & 8 $\times$ 8 & \textbf{0.96} & \textbf{35.35} & \textbf{0.06} & \textbf{24.8} &  389.8 \tabularnewline
& \textbf{OTP-DAG} (Ours)  & 8 $\times$ 8 & 0.94 & 32.56 & 0.08 & 25.2 & \textbf{462.8} \tabularnewline
\midrule
CELEBA & \textbf{SQ-VAE} & 16 $\times$ 16 & 0.88 & \textbf{31.05} & 0.12 &  14.8 & 427.8 \tabularnewline
& \textbf{OTP-DAG} (Ours)  & 16 $\times$ 16  & 0.88 & 29.77 & \textbf{0.11} & \textbf{13.1} & \textbf{487.5} \tabularnewline
\bottomrule
\end{tabular}
}
\label{tab:sqvae}
\end{table*}

\begin{figure}[hbt!]
    \centering
    \includegraphics[width=0.90\textwidth]{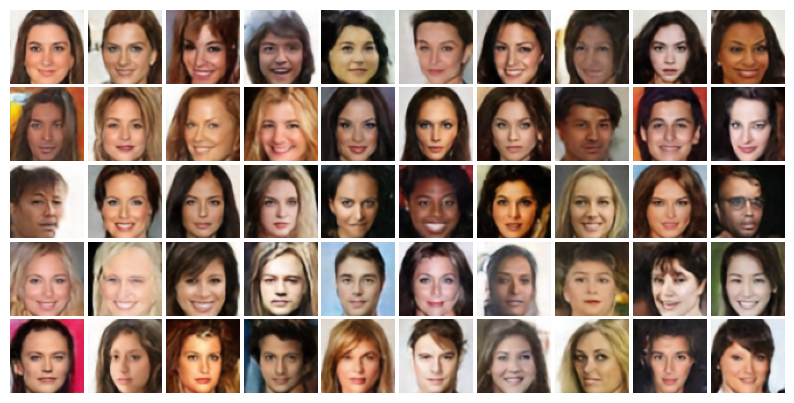}
    \caption{Generated images from the discrete representations of OTP-DAG on CelebA dataset.}
    \label{fig:quali_celeba}
\end{figure}



\begin{figure*}[ht!]
\centering
     \begin{subfigure}
         \centering\includegraphics[width=0.70\linewidth]{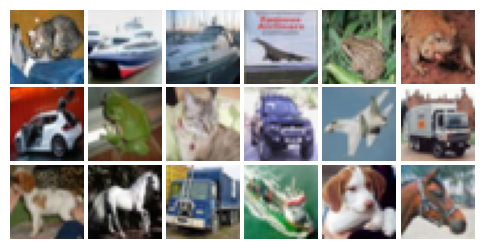}
         \caption{Original CIFAR10 images.}
         \label{fig:cifar-original}
     \end{subfigure} ~
     \begin{subfigure}
         \centering
         \includegraphics[width=0.70\linewidth]{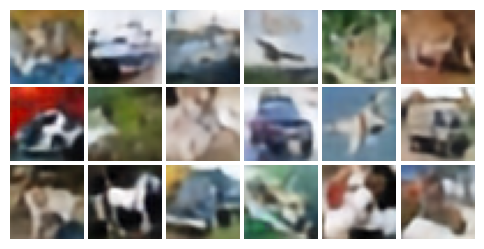}
         \caption{Random reconstructed images by VQ-VAE from CIFAR10 dataset.}
         \label{fig:cifar-vqvae}
     \end{subfigure} ~
     \begin{subfigure}
         \centering
         \includegraphics[width=0.70\linewidth]{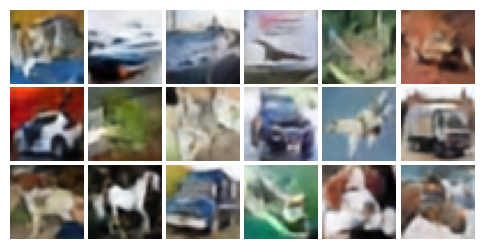}
         \caption{Random reconstructed images by OTP-DAG from CIFAR10 dataset.}
         \label{fig:cifar-otp}
     \end{subfigure}
\label{fig:quali_cifar10}   
\end{figure*}

\end{document}